\documentclass{article} 
\usepackage{iclr2027_conference,times}

\usepackage{amsmath,amsfonts,bm}

\def\eqref#1{equation~\ref{#1}}

\def\1{\bm{1}}

\DeclareMathAlphabet{\mathsfit}{\encodingdefault}{\sfdefault}{m}{sl}
\SetMathAlphabet{\mathsfit}{bold}{\encodingdefault}{\sfdefault}{bx}{n}

\renewcommand{\eqref}[1]{(\ref{#1})}

\usepackage[hypertexnames=false,hidelinks]{hyperref}
\usepackage{url}
\usepackage{amssymb,amsthm,mathtools}
\usepackage{booktabs,array}
\usepackage{enumitem}
\usepackage{graphicx}
\usepackage{xcolor}

\newtheorem{theorem}{Theorem}[section]
\newtheorem{lemma}[theorem]{Lemma}
\newtheorem{proposition}[theorem]{Proposition}
\newtheorem{corollary}[theorem]{Corollary}
\theoremstyle{definition}
\newtheorem{assumption}[theorem]{Assumption}
\newtheorem{definition}[theorem]{Definition}
\theoremstyle{remark}

\newcommand{\ip}[2]{\langle #1,#2\rangle}
\newcommand{\norm}[1]{\lVert #1\rVert}
\newcommand{\wstar}{\mathbf{w}_{*}}

\DeclareMathOperator{\conv}{conv}
\newcommand{\bgl}[1]{\mathopen{\smash{\bigg#1}}}
\newcommand{\bgr}[1]{\mathclose{\smash{\bigg#1}}}

\title{Improved Convergence of Large Stepsize \\ Gradient Descent for Logistic Regression}

\author{Xiaochuan Gong \& Ang Li \\
University of Maryland, College Park\\
}

\iclrfinalcopy 
\begin{document}

\maketitle

\begin{abstract}
We study gradient descent (GD) with a large constant stepsize for logistic regression on linearly separable data. Existing analysis shows an accelerated rate of $\widetilde{O}(1/\sqrt{\epsilon})$ to reach loss $\epsilon$ with an aggressive stepsize, although the loss may initially oscillate. Tighter control of the oscillatory dynamics has been available only for two-dimensional data. We prove a substantially faster rate in arbitrary dimension: GD with a large stepsize $\eta=1/\epsilon$ reaches loss $\epsilon$ within $O(\ln^{p}(1/\epsilon))$ steps, where $p$ depends only on the margin and the rank of the data. Our proof improves the bound on the transition time of GD from the oscillatory to the stable phase, after which the loss decreases monotonically. We split the oscillatory phase into recursively nested intervals. The margin and the rank bound the nesting depth, and a counting argument bounds the number of intervals at each depth, together yielding the polylogarithmic step complexity.
\end{abstract}

\section{Introduction}
\label{sec:intro}

Gradient descent (GD) and its variants are the most widely used optimization
methods in modern machine learning. To minimize
an objective $L:\mathbb{R}^d\to\mathbb{R}$, GD starts from an initialization
$\mathbf{w}_0\in\mathbb{R}^d$ and iterates
\begin{equation}
 \mathbf{w}_{t+1}:=\mathbf{w}_t-\eta_t\nabla L(\mathbf{w}_t),
 \qquad t=0,1,2,\ldots,
 \tag{GD}\label{eq:gd}
\end{equation}
where $\mathbf{w}_t\in\mathbb{R}^d$ is the trainable parameter and
$(\eta_t)_{t\ge0}$ are the stepsizes, which may in general vary across
iterations. In this work we study GD with a constant stepsize
$\eta_t\equiv\eta$ that is \emph{large}, by which we mean a stepsize so large
that the loss $L(\mathbf{w}_t)$ need not decrease monotonically along the
iterates.

Classical optimization theory studies GD through the \emph{descent lemma}: if
$L$ is $\beta$-smooth and $\eta_t<2/\beta$, then every step of GD
decreases the loss, and convergence rates follow by summing the per-step
decrease \citep{nesterov2018lectures}. The vast majority of convergence
guarantees for GD, for convex and nonconvex objectives alike, are built on this
monotone decrease and therefore require a small stepsize
\citep{nesterov2018lectures,lan2020first}. The practice of deep learning looks
different. Neural networks are typically trained with stepsizes so large that
the training loss does not decrease monotonically but oscillates locally, while
still converging in the long run
\citep{wu2018sgd,xing2018walk,lewkowycz2020large,cohen2021gradient}, and such
large stepsizes are often necessary for good performance.
\citet{cohen2021gradient} named this regime the \emph{edge of stability}
(EoS): the sharpness of the loss---the largest eigenvalue of its
Hessian---rises during training until it reaches the stability threshold
$2/\eta$ and then hovers just above it, so that training proceeds precisely
where the descent lemma no longer applies. Understanding why GD converges, and
why it can even converge faster, without monotone descent has become a central
question in the theory of deep learning optimization.

A recent line of work shows that stepsizes beyond the descent lemma can be
provably beneficial, in that GD converges faster with stepsizes that classical
theory does not cover. For smooth convex optimization, stepsize schedules with occasional
long steps accelerate GD beyond the $O(1/t)$ rate of the descent regime
\citep{altschuler2024acceleration,grimmer2025accelerated,zhang2025anytime}.
For logistic regression on linearly separable data, \citet{wu2024large}
showed that GD with a large constant stepsize $\eta$ converges despite an
initial \emph{unstable phase} in which the loss oscillates. GD exits this
phase within $O(\eta)$ steps and then enters a \emph{stable phase} in which
the loss decreases monotonically at rate $\widetilde O(1/(\eta t))$.
Consequently, with a budget of $T$ steps, the aggressive choice
$\eta=\Theta(T)$ balances the two phases and yields an accelerated
$\widetilde O(1/T^2)$ loss. \citet{zhang2025minimax} subsequently
showed that large stepsizes that adapt to the current loss reach an
arbitrarily small loss after a short burn-in period, which is minimax optimal
among first-order methods, and \citet{crawshaw2026tight} gave a fine-grained
analysis of the oscillatory dynamics of constant stepsize GD in two
dimensions, proving that the unstable phase ends within a number of steps
that is independent of $\eta$, which is tight up to logarithmic factors.

Among these works, \citet{wu2024large} and \citet{crawshaw2026tight} are
the closest to ours, as we consider the same problem of GD with a large
constant stepsize for logistic regression on linearly separable data. In the analysis of \citet{wu2024large},
the length of the unstable phase is bounded only linearly in the stepsize.
This is the bottleneck of their rate, since it caps the stepsize that can be
used with a given budget, and they conjectured that the bound is improvable.
\citet{crawshaw2026tight} confirmed this in two dimensions, where GD with a
sufficiently large stepsize consequently reaches any target loss within a number of
steps that is independent of the target. Their analysis, however, tracks the
oscillation of the iterates across the max-margin direction, which is a
one-dimensional motion only in the plane, and whether a comparable
improvement is possible in arbitrary dimension has remained open. This
question motivates our work.

In this paper, we consider logistic regression for binary classification.
Given a dataset $(\mathbf{x}_i,y_i)_{i=1}^n\subset\mathbb{R}^d\times\{\pm1\}$,
the empirical logistic loss is
\begin{equation}
 L(\mathbf{w}):=\frac1n\sum_{i=1}^n\ln\big(1+\exp(-y_i\mathbf{x}_i^\top\mathbf{w})\big),
 \label{eq:loss}
\end{equation}
and we run \eqref{eq:gd} on $L$ with a large constant stepsize $\eta>0$ from
the initialization $\mathbf{w}_0:=\mathbf{0}$, as in \citet{wu2024large} and
\citet{crawshaw2026tight}.
We show that GD with a large constant stepsize achieves a substantially
faster convergence rate for logistic regression on linearly separable data in
arbitrary dimension. Our contributions are as follows.
\begin{itemize}[leftmargin=*]
 \item \textbf{An improved convergence rate.}
 We prove that GD with a large constant stepsize $\eta=1/\epsilon$ reaches
 loss $\epsilon$ within $O(\ln^{p}(1/\epsilon))$ steps, where the exponent
 $p$ depends only on the margin and the rank of the data. For fixed problem dependent
 parameters, this improves the $\widetilde O(1/\sqrt{\epsilon})$ step
 complexity of \citet{wu2024large} to polylogarithmic in $1/\epsilon$ (see
 Table~\ref{tab:comparison}). At the core of this result is a new bound on
 the transition time from the unstable to the stable phase. We show that the
 unstable phase, if GD ever enters it, ends within $O((\ln\eta)^{p})$
 steps, in contrast to the
 $O(\eta)$ bound of \citet{wu2024large}.
 \item \textbf{New proof techniques.}
 Our proof is based on a counting argument that differs from both existing
 analyses. It neither relies on the time-averaged gradient potential as in
 \citet{wu2024large}, which leads to a bound linear in $\eta$, nor tracks the
 oscillations of the iterates as in \citet{crawshaw2026tight}, which
 restricts the analysis to two dimensions. Specifically, (i) we split the unstable phase into
 recursively nested intervals whose average update directions are pairwise
 separated; (ii) a geometric capacity lemma shows that only a limited number
 of pairwise separated directions can coexist, which bounds the nesting depth
 in terms of the margin and the rank; and (iii) a counting argument
 bounds the number of intervals at each depth with a factorial saving, which
 yields the polylogarithmic bound.
\end{itemize}

\paragraph{Notation.}
We use $\ip{\cdot}{\cdot}$ and $\norm{\cdot}$ to denote the Euclidean inner
product and norm, respectively. We write $\conv S$ for the convex hull of a set $S$, and
$A\mathbin{\dot\cup}B$ for the union of disjoint sets $A$ and $B$. We use
the standard big-$O$ and big-$\Omega$ notation, with $\widetilde O(\cdot)$
and $\widetilde\Omega(\cdot)$ hiding polylogarithmic factors.

\subsection{Related Work}
\label{sec:related}

\paragraph{Logistic regression.}
With a small stepsize obeying the descent lemma, GD on linearly separable
data converges in direction to the max-margin classifier
\citep{soudry2018implicit,ji2019nonseparable}, where the latter also covers
non-separable data; these results have been extended to other optimization
geometries and to stochastic gradient descent
\citep{gunasekar2018characterizing,nacson2019stochastic}.
Faster margin maximization is possible with stepsizes that grow with the inverse
loss \citep{nacson2019convergence,ji2021characterizing} or with momentum
\citep{ji2021fast,wang2022accelerated}.

For GD with a large constant stepsize, \citet{wu2023implicit} proved
convergence and asymptotic margin maximization for arbitrarily large
stepsizes on non-degenerate data, with constants exponential in $\eta$, and
\citet{wu2024large} removed the exponential dependence and established the
two-phase analysis and the $\widetilde O(1/T^2)$ acceleration discussed
above. Their analysis has since been extended to two-layer networks
\citep{cai2024large}, regularized logistic regression
\citep{wu2025regularized}, local GD \citep{crawshaw2025constant}, and adaptive
stepsizes, for which \citet{zhang2025minimax} proved a $1/\gamma^2$ step
complexity that is minimax optimal among batch first-order methods
\citep[see also][]{kornowski2024oracle}; for non-separable data, GD with a
large stepsize need not converge to the minimizer and can instead settle into
a cycle \citep{meng2024gradient,meng2025sphere}; \citet{tyurin2025logistic}
related GD with very large stepsizes to the perceptron algorithm. Most
relevant to us, \citet{crawshaw2026tight} tightened the two-phase analysis of
\citet{wu2024large} in two dimensions and proved a nearly matching lower bound. Our
work is directly motivated by \citet{wu2024large} and
\citet{crawshaw2026tight}; we improve the arbitrary-dimensional analysis of
the former and complement the two-dimensional analysis of the latter.

\paragraph{Edge of stability.}
Neural networks are commonly trained with stepsizes at which the loss
oscillates yet converges in the long run
\citep{wu2018sgd,xing2018walk,lewkowycz2020large}, a phenomenon
\citet{cohen2021gradient} named the edge of stability;
\citet{cohen2022adaptive} observed analogous behavior for adaptive methods.
A large body of theory explains the mechanism of EoS in concrete models
\citep{zhu2023understanding,ahn2023learning,kreisler2023gradient,chen2024stability,lyu2022understanding,even2023sgd,andriushchenko2023sgd,lu2024benign,wang2022large,chen2023beyond}
and characterizes the behavior of GD under instability and how the
sharpness evolves during training
\citep{kong2020stochasticity,ahn2022understanding,arora2022understanding,ma2022multiscale,wang2022analyzing,damian2023self,cohen2025understanding}. Instead of explaining the
mechanism of EoS in deep networks, we study the optimization benefits of GD
operating in the unstable regime, in a natural learning problem and under
standard assumptions.

\paragraph{Aggressive stepsizes in convex optimization.}
For smooth convex optimization, a recent line of work shows that stepsize
schedules with occasional long steps, which may temporarily increase the loss,
provably improve the convergence of GD. Long-step schedules first gave
constant-factor gains \citep{altschuler2018greed,grimmer2024provably} or
gains for objectives with multiple scales \citep{kelner2022big}, and the
silver stepsize schedule and its successors accelerate GD beyond the
$O(1/t)$ rate of the descent regime
\citep{altschuler2024acceleration,grimmer2025accelerated,zhang2025anytime}.
These works design stepsize sequences for function classes whose minimizers are
attained, and the acceleration is driven by the schedule. In contrast, we
study a large constant stepsize on a problem whose infimum is not
attained, where the acceleration arises from the dynamics of GD itself.

\section{Main Results}
\label{sec:main}

We assume that the data are bounded and linearly separable
\citep{novikoff1962convergence,wu2024large}.
\begin{assumption}[Bounded and separable data]
\label{ass:data}
Assume the training data $(\mathbf{x}_i,y_i)_{i=1}^n$ satisfies
\begin{enumerate}[label=\Alph*.,leftmargin=*,nosep]
 \item for every $i=1,\ldots,n$, $\norm{\mathbf{x}_i}\le1$ and $y_i\in\{\pm1\}$;
 \item there is a margin $\gamma\in(0,1]$ and a unit vector $\wstar$ such that
 $\ip{y_i\mathbf{x}_i}{\wstar}\ge\gamma$ for every $i=1,\ldots,n$.
\end{enumerate}
\end{assumption}
We write
$r:=\dim\operatorname{span}\{y_1\mathbf{x}_1,\ldots,y_n\mathbf{x}_n\}\le\min\{n,d\}$
for the rank of the data.

Under Assumption~\ref{ass:data}, the logistic loss \eqref{eq:loss} is convex
and $\beta$-smooth with $\beta=1/4$. For stepsizes $\eta<2/\beta=8$, the
descent lemma guarantees monotone decrease of the loss, and classical theory
shows that GD reaches loss $\epsilon$ within $\widetilde O(1/\epsilon)$ steps
for fixed $n$ and $\gamma$ \citep{nesterov2018lectures,ji2019nonseparable}. With a larger stepsize the loss may
initially oscillate, yet for suitable choices GD converges even faster. \citet{wu2024large} showed
that GD with a large stepsize $\eta=\Theta(1/\sqrt{\epsilon})$ reaches
loss $\epsilon$ within $\widetilde O(1/\sqrt{\epsilon})$ steps. In this work, we show
that GD with a large stepsize $\eta=1/\epsilon$ reaches loss $\epsilon$
within $O(\ln^{p}(1/\epsilon))$ steps, where $p$ depends only on
the margin and the rank of the data.

Following \citet{wu2024large}, we say that GD is in the \emph{stable phase}
at step $t$ if the loss decreases monotonically from step $t$ onwards, and in
the \emph{unstable phase} otherwise. If $L(\mathbf{w}_s)\le1/\eta$ at some
step $s$, then GD is in the stable phase at step $s$
\citep[Theorem~1]{wu2024large}. We therefore define the \emph{phase transition
time}
\begin{equation}
 \tau:=\min\{t\ge0:\ L(\mathbf{w}_t)\le1/\eta\},
 \label{eq:tau}
\end{equation}
so that the loss may oscillate before step $\tau$ and decreases
monotonically afterwards.

To state the theorem, define the depth $\nu$, the exponent $p$, the
prefactor $C$, and the stepsize threshold $\eta_0$ (whose roles in the proof are
explained in Section~\ref{sec:analysis}) by
\begin{align}
 \nu&:=\max\bgl\{0,\ \min\bgl\{r,\ \bgl\lfloor\gamma^{-2}+\frac{\gamma^{-2}-1}{h-1}\bgr\rfloor\bgr\}-2\bgr\},
 \quad\text{where}\ \ h:=\lceil\ln(16\gamma^{-2})\rceil,
 \label{eq:def-depth}\\
 p&:=\begin{cases}
 \nu+\bgl\lceil\dfrac{\nu\ln(18h\gamma^{-2}/\nu)}{\ln(\max\{2,\nu\})}\bgr\rceil
 & \text{if }\nu\ge1,\\[2pt]
 0 & \text{if }\nu=0,
 \end{cases}
 \label{eq:def-exponent}\\
 C&:=80\gamma^{-4}h\big(n+h\ln(4\gamma^{-2})\big),
 \label{eq:def-prefactor}\\
 \eta_0&:=\big\lceil128\,K(1+\ln K)\big\rceil,
 \quad\text{where}\ \ K:=\exp(4)+h\gamma^{-2}+p+\ln(2\gamma^{-2}C).
 \label{eq:def-onset}
\end{align}
We now present an improved analysis for GD with a large constant stepsize
for logistic regression on linearly separable data. Our main theorem
improves the dependence of the phase transition time on $\eta$ from linear
\citep{wu2024large} to polylogarithmic in arbitrary dimension.
\begin{theorem}
\label{thm:main}
Consider \eqref{eq:gd} with a constant stepsize $\eta>0$ and initialization
$\mathbf{w}_0=\mathbf{0}$ for logistic regression \eqref{eq:loss} under
Assumption~\ref{ass:data}, and let $p$, $C$, and $\eta_0$ be as in
\eqref{eq:def-depth}--\eqref{eq:def-onset}.
\begin{enumerate}[label=(\roman*),leftmargin=*]
 \item \textbf{Phase transition.} If $0<\gamma<1/\sqrt2$ and
 $\eta\ge\eta_0$, then
 \begin{equation}
  \tau\le C\,(\ln\eta)^{p};
  \label{eq:main-bound}
 \end{equation}
 if $\gamma\ge1/\sqrt2$ and $\eta\ge64n^2$, then $\tau\le1$.
 \item \textbf{The stable phase} \citep[Theorem~1]{wu2024large}. The loss
 decreases monotonically from step $\tau$ onwards, that is,
 $L(\mathbf{w}_{t+1})\le L(\mathbf{w}_t)$ for every $t\ge\tau$.
\end{enumerate}
\end{theorem}

The proof of Theorem~\ref{thm:main} is deferred to
Section~\ref{sec:analysis} and Appendices~\ref{app:constants}
and~\ref{app:proofs}.
Theorem~\ref{thm:main} describes both phases of GD with a large constant
stepsize. Part~(i) bounds the length of the unstable phase. Once the
stepsize exceeds the threshold $\eta_0$, GD reaches loss $1/\eta$ within
$C(\ln\eta)^{p}$ steps, and when $\gamma\ge1/\sqrt2$ a single step suffices.
This improves the linear dependence on $\eta$ of the previous
phase transition time bound \citep{wu2024large} to a polylogarithmic one in
arbitrary dimension. Part~(ii) describes the stable phase and restates the
analysis of \citet{wu2024large}. From step $\tau$ onwards the loss
decreases monotonically, and Appendix~\ref{app:stable} recalls the rate at
which it continues to decrease.

\paragraph{Parameter dependence.}
For $0<\gamma<1/\sqrt2$, the quantities in
\eqref{eq:def-depth}--\eqref{eq:def-onset} satisfy
\begin{equation*}
 C=O\big(\gamma^{-4}\ln(\gamma^{-2})\,[\,n+\ln^{2}(\gamma^{-2})\,]\big),
 \quad\
 \eta_0=\widetilde O\big(\gamma^{-2}+\ln n\big),
 \quad\
 p=O\big(\min\{r,\gamma^{-2}\}\ln(\gamma^{-2})\big).
\end{equation*}
Thus the prefactor is polynomial in $n$ and $1/\gamma$, while the threshold
$\eta_0$ is polynomial in $1/\gamma$ and only polylogarithmic in $n$. When $r\le2$, we
have $\nu=p=0$, and the bound \eqref{eq:main-bound} becomes
$\tau\le C$, which does not depend on $\eta$. When the rank satisfies
the condition \eqref{eq:rank-condition} in Appendix~\ref{app:parameters},
the exponent satisfies $p=(1+o(1))\gamma^{-2}$ as $\gamma\to0$.
All hidden constants are absolute, and these bounds are proved in
Proposition~\ref{prop:parameters}.

\paragraph{Comparisons with prior results.}
\citet{wu2024large} proved the phase transition time bound
\begin{equation}
 \tau\le\frac{60}{\gamma^2}
 \max\bgl\{\eta,\ n,\ e,\ \frac{\eta+n}{\eta}\ln\frac{\eta+n}{\eta}\bgr\}
 \label{eq:wu-tau}
\end{equation}
for every $\eta>0$ in arbitrary dimension. The bound \eqref{eq:wu-tau} grows
linearly with the stepsize, whereas the bound \eqref{eq:main-bound} grows
only polylogarithmically, and does not grow at all when $r\le2$; for large
stepsizes, \eqref{eq:main-bound} is therefore smaller by a factor almost
linear in $\eta$. The price is the stepsize threshold, as
\eqref{eq:main-bound} holds only for $\eta\ge\eta_0$ while
\eqref{eq:wu-tau} holds for every $\eta>0$. This threshold is mild, however, and the two-dimensional bound
of \citet{crawshaw2026tight} requires $\eta\ge n$.

The threshold $1/\sqrt2$ in Theorem~\ref{thm:main}(i) marks a change in the
data geometry. For $\gamma\ge1/\sqrt2$ the signed examples have pairwise
nonnegative inner products, and every GD step adds a nonnegative combination
of them, so no margin ever decreases and the loss decreases monotonically
from initialization for every $\eta>0$. There is no unstable phase in this
regime, and the second case of Theorem~\ref{thm:main}(i) adds that a single
step with $\eta\ge64n^2$ already brings the loss below $1/\eta$.
\citet{crawshaw2026tight} observe the same effect and exclude this regime
from their lower bound.

The phase transition bound in part~(i) and the monotonicity in part~(ii) of
Theorem~\ref{thm:main} combine to give the step complexity of GD with a large
constant stepsize, in terms of a target loss $\epsilon$ and of a budget of
$T$ steps. The proof is deferred to Appendix~\ref{app:corollaries}.

\begin{corollary}[Complexity of GD with a large constant stepsize]
\label{cor:accuracy}
Under the setting of Theorem~\ref{thm:main}, set $C:=1$ and $p:=0$ if
$\gamma\ge1/\sqrt2$, and define the thresholds
\[
 \epsilon_0:=
 \begin{cases}
  1/\eta_0 & \text{if }\gamma<1/\sqrt2,\\
  1/(64n^2) & \text{if }\gamma\ge1/\sqrt2,
 \end{cases}
 \qquad
 T_0:=\big\lceil C(\ln\eta_0)^{p}\big\rceil.
\]
\begin{enumerate}[label=(\roman*),leftmargin=*]
 \item \textbf{Target loss.} For every $0<\epsilon\le\epsilon_0$, GD with
 the constant stepsize $\eta=1/\epsilon$ satisfies
 $L(\mathbf{w}_t)\le\epsilon$ for some $t\le C\ln^{p}(1/\epsilon)$, and the
 loss decreases monotonically from step $t$ onwards.
 \item \textbf{Budget.} If $p\ge1$, then for every budget
 $T\ge T_0$, GD with the constant stepsize $\eta=\exp\big((T/C)^{1/p}\big)$
 satisfies $L(\mathbf{w}_T)\le\exp\big(-(T/C)^{1/p}\big)$. If $p=0$, then
 for every $0<\epsilon\le\epsilon_0$, GD with the constant stepsize
 $\eta=1/\epsilon$ satisfies $L(\mathbf{w}_T)\le\epsilon$ for every
 $T\ge C$.
\end{enumerate}
\end{corollary}

\begin{table}[t]
\centering
\caption{Step complexity of constant stepsize GD for logistic regression on
linearly separable data. Each complexity holds for all sufficiently small
$\epsilon$.}
\label{tab:comparison}
\vspace{0.5em}
\footnotesize
\setlength{\tabcolsep}{3pt}
\renewcommand{\arraystretch}{1.4}
\begin{tabular*}{\textwidth}{@{\extracolsep{\fill}}llll@{}}
\toprule
 & Complexity
 & Stepsize
 & Setting \\
\midrule
\citet[Corollary~2]{wu2024large}
 & $\widetilde O\big(1/(\gamma^{2}\sqrt{\epsilon})\big)$
 & $\eta=\Theta(1/\sqrt{\epsilon})$
 & any $d$ \\
\citet[Theorem~2.1]{crawshaw2026tight}
 & $O\big(n/\gamma+\ln(1/\gamma)/\gamma^{2}\big)$
 & $\eta\ge\Omega\big(1/(\epsilon(\gamma n+1))\big)$
 & $d=2$ \\
\midrule
Corollary~\ref{cor:accuracy}
 & $O\big(n\gamma^{-4}\ln^{3}(e/\gamma)\ln^{p}(1/\epsilon)\big)$
 & $\eta=1/\epsilon$
 & any $d$ \\
\bottomrule
\end{tabular*}
\end{table}

Corollary~\ref{cor:accuracy} shows that, for fixed $n$ and $\gamma$, the
step complexity improves from $\widetilde O(1/\sqrt{\epsilon})$
\citep{wu2024large} to $O(\ln^{p}(1/\epsilon))$, that is, from polynomial to
polylogarithmic in $1/\epsilon$. See Table~\ref{tab:comparison} for a comparison
with prior results. When $p=0$, in
particular when $r\le2$ or $\gamma\ge1/\sqrt2$, the complexity does not
depend on $\epsilon$, as \citet{crawshaw2026tight} showed for $d=2$. Given a
budget of $T$ steps, the loss improves from $\widetilde O(1/T^2)$
\citep{wu2024large} to $\exp(-(T/C)^{1/p})$ when $p\ge1$.

\paragraph{Limitations.}
Our bounds are not tight in several respects. First, the prefactor $C$ is
not optimized; in two dimensions it is weaker than the nearly tight bound of
\citet{crawshaw2026tight}. Second,
our results apply only above the stepsize threshold $\eta_0$, and we claim
nothing for moderate stepsizes between the descent regime and $\eta_0$.
Third, we prove no lower bounds, and this paper does not determine whether
the exponent $p$ or the dependence on $\eta$ in \eqref{eq:main-bound} is
optimal.

\section{Analysis}
\label{sec:analysis}

In this section we prove the phase transition bound of
Theorem~\ref{thm:main}(i) for $0<\gamma<1/\sqrt2$; the large-margin case is
elementary and is proved in Appendix~\ref{app:large-margin}, and part~(ii)
is recalled in Appendix~\ref{app:stable}. Throughout,
$\mathbf{z}_i:=y_i\mathbf{x}_i$ are the signed examples, $\eta\ge\eta_0$,
$\Lambda:=\ln\eta$, and $h$, $\nu$, $p$, $C$, and $\eta_0$ are as in
\eqref{eq:def-depth}--\eqref{eq:def-onset}. Rather than averaging a
potential along the trajectory or tracking the oscillations of the iterates,
we count steps. Steps at which some signed margin is negative are few. The
remaining steps are organized into nested intervals (see
\textbf{Figure~\ref{fig:hierarchy}} for an
illustration), each carrying reference
directions that constrain every update in it. The margin and the rank
bound the depth of this hierarchy, and the cumulative update coefficients of
sibling intervals are separated on a logarithmic scale, which lets us count
the intervals at each depth with a factorial saving. Adding up the lengths of
the intervals proves the bound.

\subsection{Update geometry before the phase transition}
\label{sec:geometry}

For each step, define the logistic weights, the update coefficient, and
the update direction
\begin{equation}
 q_{i,t}:=\frac{1}{1+\exp(\ip{\mathbf{z}_i}{\mathbf{w}_t})},
 \qquad
 Q_t:=\sum_{i=1}^{n}q_{i,t},
 \qquad
 \mu_t:=\frac{\eta Q_t}{n},
 \qquad
 \mathbf{v}_t:=\frac{1}{Q_t}\sum_{i=1}^{n}q_{i,t}\mathbf{z}_i.
 \label{eq:weights}
\end{equation}
In this notation, the update \eqref{eq:gd} becomes
\begin{equation}
 \mathbf{w}_{t+1}-\mathbf{w}_t=\mu_t\mathbf{v}_t,
 \qquad
 \mathbf{v}_t\in\conv\{\mathbf{z}_1,\ldots,\mathbf{z}_n\},
 \qquad
 \norm{\mathbf{v}_t}\le1,
 \qquad
 \ip{\mathbf{v}_t}{\wstar}\ge\gamma.
 \label{eq:update}
\end{equation}
Thus $\mu_t$ is the coefficient of the step and $\mathbf{v}_t$ its
direction, an average of the signed examples. The trajectory estimates
below hold on the finite horizon
\begin{equation}
 \widehat T:=\bgl\lfloor\frac{\gamma^{2}\exp(\eta/4)}{\eta}\bgr\rfloor,
 \qquad
 U:=\min\{\tau,\widehat T\},
 \label{eq:horizon}
\end{equation}
and we count the steps $0,1,\ldots,U-1$; at the end of the proof we show
that $U<\widehat T$, which forces $U=\tau$. We split these steps
according to the signs of the margins, into
$\mathcal T_-:=\{t<U:\ \ip{\mathbf{z}_i}{\mathbf{w}_t}<0\text{ for some }i\}$
and $\mathcal T_+:=\{0,1,\ldots,U-1\}\setminus\mathcal T_-$. The next lemma
bounds the total update coefficient before $U$, which shows that
$\mathcal T_-$ is small, and states two facts about the steps in
$\mathcal T_+$: each of them carries a coefficient of at least $1/2$, and
the iterate is not yet far along its own update direction.

\begin{lemma}[Trajectory estimates]
\label{lem:basic}
We have
\begin{equation}
 \sum_{t<U}\mu_t\le2\eta\gamma^{-2},
 \qquad
 |\mathcal T_-|<4n\gamma^{-2},
 \label{eq:budget}
\end{equation}
and at every $t\in\mathcal T_+$,
\begin{equation}
 \mu_t\ge\frac{\eta}{2}L(\mathbf{w}_t)>\frac12,
 \qquad
 0\le\ip{\mathbf{w}_t}{\mathbf{v}_t}<\Lambda.
 \label{eq:pointwise}
\end{equation}
\end{lemma}

\begin{proof}
The first bound in \eqref{eq:budget} follows from the split optimization
technique introduced by \citet{wu2024large} on the horizon \eqref{eq:horizon}, for which
$\eta\ge4$ suffices (Appendix~\ref{app:trajectory}). At
$t\in\mathcal T_-$ some weight exceeds $1/2$, so $\mu_t>\eta/(2n)$, and the
second bound follows from the first. Now let $t\in\mathcal T_+$, write
$\ell(u):=\ln(1+e^{-u})$, and let $m_i:=\ip{\mathbf{z}_i}{\mathbf{w}_t}\ge0$.
For $u\ge0$ and $a:=e^{-u}\in(0,1]$,
\[
 \frac{1}{1+e^{u}}=\frac{a}{1+a}\ge\frac a2\ge\frac12\ln(1+a)
 =\frac12\ell(u),
\]
so $Q_t/n\ge L(\mathbf{w}_t)/2$, which gives $\mu_t\ge\eta L(\mathbf{w}_t)/2$;
this exceeds $1/2$ because $t<\tau$. For the radial bound, convexity
of $\ell$ and $\ell'(m_i)=-q_{i,t}$ give
$\ell(\Lambda)\ge\ell(m_i)+q_{i,t}(m_i-\Lambda)$, and summing over $i$,
\[
 Q_t\big(\ip{\mathbf{w}_t}{\mathbf{v}_t}-\Lambda\big)
 =\sum_{i=1}^{n}q_{i,t}(m_i-\Lambda)
 \le n\ell(\Lambda)-nL(\mathbf{w}_t)<0,
\]
because $\ell(\Lambda)=\ln(1+1/\eta)<1/\eta<L(\mathbf{w}_t)$. Nonnegativity
holds because every $m_i\ge0$.
\end{proof}

\subsection{A bounded-depth hierarchy of intervals}
\label{sec:hierarchy}

\begin{figure}[t]
\centering
\includegraphics[width=\textwidth]{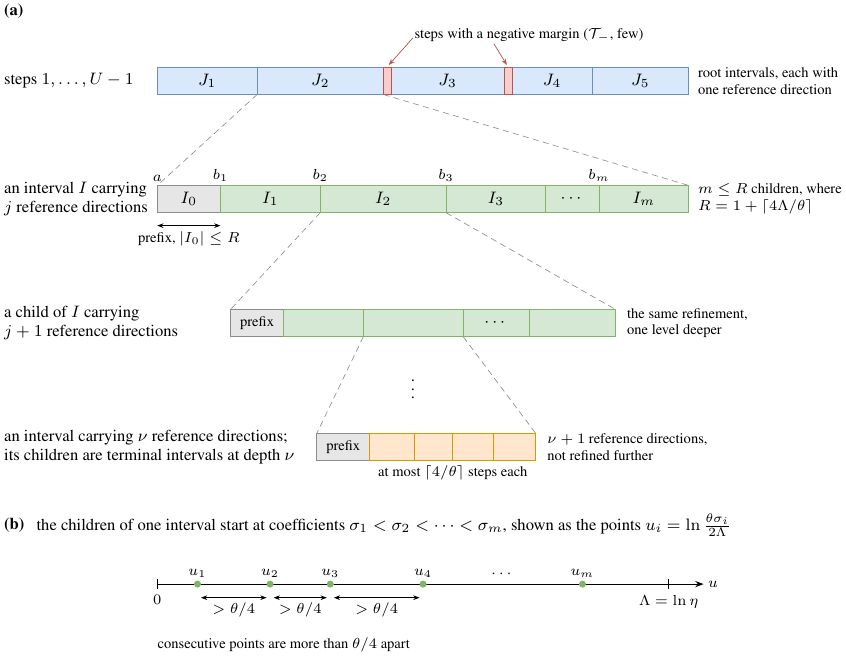}
\caption{The hierarchy of intervals. (a) The steps in $\mathcal T_+\setminus\{0\}$ are partitioned into root
intervals, each carrying one reference direction. An interval carrying $j$ reference
directions is split into a prefix of at most $R$ steps and at most $R$
children, and each child inherits the $j$ directions and gains one more
(Lemma~\ref{lem:refine}). An interval reached after $\nu$ refinements carries
$\nu+1$ reference directions; it is terminal, is not refined further, and has
at most $\lceil4/\theta\rceil$ steps (Lemma~\ref{lem:terminal}). (b) The
children of one interval start at coefficients $\sigma_i$, measured from the
start of the parent, and the points $u_i=\ln(\theta\sigma_i/(2\Lambda))$ are
more than $\theta/4$ apart, so a sum over the children can be compared with
an integral along this axis. Lemma~\ref{lem:factorial} applies this at every
level of the hierarchy, which yields the factorial saving.}
\label{fig:hierarchy}
\end{figure}

This section builds the hierarchy of intervals and bounds its depth. See
Figure~\ref{fig:hierarchy} for an illustration. Fix
$\theta:=\gamma^{2}/h$. Every interval below is a set of consecutive steps
in $\mathcal T_+$. We say that vectors $\mathbf{a}_1,\ldots,\mathbf{a}_j$
are \emph{reference directions} for an interval $I$ if
\begin{equation}
 \mathbf{a}_i\in\conv\{\mathbf{z}_1,\ldots,\mathbf{z}_n\},
 \qquad
 \ip{\mathbf{a}_i}{\mathbf{a}_k}\le\theta\ \ (i\ne k),
 \qquad
 \ip{\mathbf{a}_i}{\mathbf{v}_t}\le\theta\ \ (t\in I).
 \label{eq:references}
\end{equation}
A reference direction is an earlier average update direction from which
every update in $I$ has turned away. These constraints are local to one
branch of the construction.

Both constructions in the next lemma freeze a running average of update
directions and start a new interval whenever the average has moved by more
than $\theta/2$. Roots use the average from the origin,
$\boldsymbol{\xi}_t:=\mathbf{w}_t/S_t$ $(t\ge1)$ with $S_t:=\sum_{k<t}\mu_k$;
refinement uses the average since the start of the parent.

\begin{lemma}[Root partition and refinement]
\label{lem:refine}
The following statements hold with $R:=1+\lceil4\Lambda/\theta\rceil$.

\textup{(i)} The steps in $\mathcal T_+\setminus\{0\}$ can be partitioned
into at most $|\mathcal T_-|+1+\lceil4\ln(4\gamma^{-2})/\theta\rceil$
consecutive intervals, each carrying one reference direction.

\textup{(ii)} Every interval $I=\{a,\ldots,b\}\subseteq\mathcal T_+$
carrying $j$ reference directions admits a partition
\begin{equation}
 I=I_0\mathbin{\dot\cup}I_1\mathbin{\dot\cup}\cdots\mathbin{\dot\cup}I_m,
 \qquad |I_0|\le R,\qquad m\le R,
 \label{eq:R}
\end{equation}
into a prefix $I_0$ and consecutive children $I_1,\ldots,I_m$. Each child
inherits the
$j$ reference directions of $I$ and gains one new reference direction
satisfying \eqref{eq:references}. If $b_i:=\min I_i$ and
$\sigma_i:=\sum_{t=a}^{b_i-1}\mu_t$, then
\begin{equation}
 \sigma_i\ge\frac{2\Lambda}{\theta}\quad(1\le i\le m),
 \qquad
 \ln\frac{\sigma_{i+1}}{\sigma_i}>\frac\theta4\quad(1\le i<m).
 \label{eq:separation}
\end{equation}
The case $m=0$ is allowed, with $I_0=I$.
\end{lemma}

\begin{proof}[Proof sketch]
\emph{Root partition.} The vector $\boldsymbol{\xi}_{t+1}$ is a convex
combination of $\boldsymbol{\xi}_t$ and $\mathbf{v}_t$ with weights
$S_t/S_{t+1}$ and $\mu_t/S_{t+1}$, and both vectors lie in the convex hull
of the signed examples, so
\begin{equation}
 \norm{\boldsymbol{\xi}_{t+1}-\boldsymbol{\xi}_t}
 \le\frac{2\mu_t}{S_{t+1}}\le2\ln\frac{S_{t+1}}{S_t},
 \qquad
 \sum_{t=1}^{U-1}\norm{\boldsymbol{\xi}_{t+1}-\boldsymbol{\xi}_t}
 \le2\ln\frac{S_U}{S_1}\le2\ln(4\gamma^{-2}),
 \label{eq:path}
\end{equation}
by $S_1=\mu_0=\eta/2$ and \eqref{eq:budget}. The set
$\mathcal T_+\setminus\{0\}$ consists of at most $|\mathcal T_-|+1$ maximal
runs of consecutive steps. Cut each run whenever $\boldsymbol{\xi}_t$ has
moved by more than $\theta/2$ from its value at the start of the current
piece; by \eqref{eq:path} there are at most $4\ln(4\gamma^{-2})/\theta$ cuts
in total. A piece starting at step $c$ receives
$\mathbf{a}_1:=\boldsymbol{\xi}_c$, which lies in the convex hull, and for
every $t$ in the piece, by \eqref{eq:pointwise}, $S_t\ge S_1=\eta/2$, and
the stepsize condition $4\Lambda\le\eta\theta$,
\[
 \ip{\mathbf{a}_1}{\mathbf{v}_t}
 \le\frac{\ip{\mathbf{w}_t}{\mathbf{v}_t}}{S_t}
 +\norm{\boldsymbol{\xi}_c-\boldsymbol{\xi}_t}
 <\frac{2\Lambda}{\eta}+\frac\theta2\le\theta.
\]

\emph{Refinement.} Let $S^{(a)}_t:=\sum_{k=a}^{t-1}\mu_k$ and, for $t>a$,
let $\boldsymbol{\zeta}_t:=(\mathbf{w}_t-\mathbf{w}_a)/S^{(a)}_t$ be the
average update direction since step $a$. Every signed margin at $a$ is
nonnegative and $\mathbf{v}_t$ is a convex combination of signed examples,
so $\ip{\mathbf{w}_a}{\mathbf{v}_t}\ge0$, and \eqref{eq:pointwise} gives
\begin{equation}
 \ip{\boldsymbol{\zeta}_t}{\mathbf{v}_t}
 \le\frac{\ip{\mathbf{w}_t}{\mathbf{v}_t}}{S^{(a)}_t}
 <\frac{\Lambda}{S^{(a)}_t}\le\frac\theta2
 \qquad\text{once }S^{(a)}_t\ge\frac{2\Lambda}{\theta}.
 \label{eq:local-radial}
\end{equation}
The prefix $I_0$ consists of the steps of $I$ at which
$S^{(a)}_t<2\Lambda/\theta$; since $\mu_t>1/2$ on $\mathcal T_+$, it has at
most $\lceil4\Lambda/\theta\rceil$ steps, and $\sigma_i\ge2\Lambda/\theta$
by construction. On the rest of $I$ we cut a new child whenever
$\boldsymbol{\zeta}_t$ has moved by more than $\theta/2$ from its value at
the start of the current child. As in \eqref{eq:path},
$\norm{\boldsymbol{\zeta}_{t+1}-\boldsymbol{\zeta}_t}
\le2\ln(S^{(a)}_{t+1}/S^{(a)}_t)$, and the total movement after the prefix
is at most $2\ln(\eta\gamma^{-2}\theta/\Lambda)<2\Lambda$ by
\eqref{eq:budget}, so $m<1+4\Lambda/\theta$. The child $I_i$ receives
$\mathbf{a}_{j+1}:=\boldsymbol{\zeta}_{b_i}$, a convex combination of the
$\mathbf{v}_k$ with $a\le k<b_i$. The inner-product bounds
\eqref{eq:references} for the old reference directions pass through this
average exactly, and the bound for the new direction follows from
\eqref{eq:local-radial} and the cut rule:
\begin{equation}
\begin{aligned}
 \ip{\mathbf{a}_s}{\mathbf{a}_{j+1}}
 &=\frac1{\sigma_i}\sum_{k=a}^{b_i-1}\mu_k\ip{\mathbf{a}_s}{\mathbf{v}_k}
 \le\theta\qquad(1\le s\le j),\\
 \ip{\mathbf{a}_{j+1}}{\mathbf{v}_t}
 &\le\ip{\boldsymbol{\zeta}_t}{\mathbf{v}_t}
 +\norm{\boldsymbol{\zeta}_{b_i}-\boldsymbol{\zeta}_t}
 \le\frac\theta2+\frac\theta2=\theta
 \qquad(t\in I_i).
\end{aligned}
 \label{eq:caps}
\end{equation}
Finally, a completed child moves the average by more than $\theta/2$, and
the movement over the child is at most $2\ln(\sigma_{i+1}/\sigma_i)$, which
gives \eqref{eq:separation}. Appendix~\ref{app:refinement} gives the
cut points and the counts.
\end{proof}

Lemma~\ref{lem:refine} shows where the powers of $\Lambda=\ln\eta$ come
from. The roots use the average from the origin, whose total movement
\eqref{eq:path} does not depend on $\eta$, so the first reference direction
costs no factor of $\Lambda$. Each refinement costs one factor of $\Lambda$
through $R$, because the average since the start of the parent must first
accumulate coefficient $2\Lambda/\theta$. Since the reference directions are
convex averages, every inherited inner-product bound is preserved exactly, so
the same $\theta$ serves at every depth. The depth of the hierarchy is
bounded by the following capacity argument.

\begin{lemma}
\label{lem:capacity}
Let $\mathbf{b}_1,\ldots,\mathbf{b}_m\in\conv\{\mathbf{z}_1,\ldots,\mathbf{z}_n\}$
satisfy $\ip{\mathbf{b}_i}{\mathbf{b}_k}\le\theta$ for all $i\ne k$. Then
\begin{equation}
 m\le\min\bgl\{r,\ \bgl\lfloor\frac{1-\theta}{\gamma^{2}-\theta}\bgr\rfloor\bgr\}
 \le\nu+2.
 \label{eq:capacity}
\end{equation}
\end{lemma}

\begin{proof}
Each $\mathbf{b}_i$ satisfies $\norm{\mathbf{b}_i}\le1$ and
$\ip{\mathbf{b}_i}{\wstar}\ge\gamma$. Suppose $\sum_ic_i\mathbf{b}_i=\mathbf{0}$
with the $c_i$ not all zero. Coefficients of both signs occur, since a
nonzero nonnegative combination has positive inner product with $\wstar$.
Let $\mathbf{u}:=\sum_{c_i>0}c_i\mathbf{b}_i=\sum_{c_k<0}(-c_k)\mathbf{b}_k$,
$a:=\sum_{c_i>0}c_i$, and $b:=\sum_{c_k<0}(-c_k)$. The separator gives
$\norm{\mathbf{u}}^2\ge\gamma^2\max\{a^2,b^2\}\ge\gamma^2ab$, while pairing
the two representations gives $\norm{\mathbf{u}}^2\le\theta ab<\gamma^2ab$,
a contradiction. Hence the $\mathbf{b}_i$ are linearly independent vectors
in the span of the data, and $m\le r$. Next,
\[
 m^2\gamma^2
 \le\ip{\textstyle\sum_i\mathbf{b}_i}{\wstar}^2
 \le\big\lVert\textstyle\sum_i\mathbf{b}_i\big\rVert^2
 \le m+m(m-1)\theta,
\]
so $m(\gamma^2-\theta)\le1-\theta$. With $\theta=\gamma^2/h$ the floor in
\eqref{eq:capacity} equals
$\lfloor\gamma^{-2}+(\gamma^{-2}-1)/(h-1)\rfloor$, and \eqref{eq:def-depth}
gives the last inequality.
\end{proof}

By Lemma~\ref{lem:capacity}, at most $\nu+2$ vectors of the convex hull can
have pairwise inner products at most $\theta$. An interval carrying $\nu+1$
reference directions therefore leaves room for only one more such vector, so
any two of its update directions have inner product above $\theta$. The next
lemma shows that such an interval is short, because
$1/(\eta L(\mathbf{w}_t))$ grows by a fixed amount at each of its steps.

\begin{lemma}
\label{lem:terminal}
If an interval $I\subseteq\mathcal T_+$ carries $\nu+1$ reference
directions, then $|I|\le\lceil4/\theta\rceil$.
\end{lemma}

\begin{proof}
If two distinct steps $s,t\in I$ had $\ip{\mathbf{v}_s}{\mathbf{v}_t}\le\theta$,
then $\mathbf{v}_s$, $\mathbf{v}_t$, and the $\nu+1$ reference directions
would be $\nu+3$ vectors in the convex hull with pairwise inner products at
most $\theta$, contradicting Lemma~\ref{lem:capacity}. Hence
$\ip{\mathbf{v}_t}{\mathbf{v}_{t+1}}>\theta$ for consecutive steps in $I$.
Convexity of $L$ at $\mathbf{w}_{t+1}$, the identity
$\nabla L(\mathbf{w}_{t+1})=-(\mu_{t+1}/\eta)\mathbf{v}_{t+1}$, and
\eqref{eq:pointwise} give
\[
 L(\mathbf{w}_t)-L(\mathbf{w}_{t+1})
 \ge\ip{\nabla L(\mathbf{w}_{t+1})}{\mathbf{w}_t-\mathbf{w}_{t+1}}
 =\frac{\mu_t\mu_{t+1}}{\eta}\ip{\mathbf{v}_t}{\mathbf{v}_{t+1}}
 \ge\frac{\theta\eta}{4}L(\mathbf{w}_t)L(\mathbf{w}_{t+1}),
\]
that is, $1/(\eta L(\mathbf{w}_{t+1}))-1/(\eta L(\mathbf{w}_t))\ge\theta/4$.
Before $\tau$ every reciprocal $1/(\eta L(\mathbf{w}_t))$ is below one,
so summing over the $|I|-1$ consecutive pairs gives $(|I|-1)\theta/4<1$,
that is, $|I|\le\lceil4/\theta\rceil$.
\end{proof}

Each root carries one reference direction and each refinement adds one, so
an interval reached after $\nu$ refinements carries $\nu+1$ reference
directions. We call such an interval \emph{terminal}. It is not refined
further, and Lemma~\ref{lem:terminal} bounds its length. This is the role of
the $2$ subtracted in \eqref{eq:def-depth}, since $\nu+1$ reference
directions and two update directions are one vector more than
Lemma~\ref{lem:capacity} allows. When $\nu=0$, in particular when $r\le2$,
every root is already terminal and no refinement is needed, which is why
$p=0$ there.

\subsection{Counting intervals}
\label{sec:packing}

This section counts the intervals of the hierarchy at each depth and
bounds the total length of one tree. The construction produces one rooted
tree per root interval. A node is an interval together with its reference
directions, the children of a nonterminal interval are the
$I_1,\ldots,I_m$ of Lemma~\ref{lem:refine}(ii), and terminal intervals have no
children. Let $D_j(I)$ denote the number of descendants of $I$ at depth
$j$ below $I$, with $D_0(I)=1$. Each nonterminal interval has at most
$R\le6\Lambda/\theta$ children, as $\Lambda>1>\theta$, so
$D_j(I)\le(6\Lambda/\theta)^{j}$ for every $j$. This count ignores that the
children of one interval share its coefficient budget. By
Lemma~\ref{lem:refine}(ii), the children start at coefficients
$\sigma_1<\sigma_2<\cdots$ measured from the start of the parent, and
\eqref{eq:separation} keeps their logarithms $\theta/4$ apart. Counting the
siblings jointly along this axis yields a factorial saving.

To this end, for a nonempty interval $I$ define its coefficient and its
logarithmic size
\begin{equation}
 \mu(I):=\sum_{t\in I}\mu_t,
 \qquad
 \chi(I):=\max\bgl\{0,\ \ln\frac{\theta\mu(I)}{2\Lambda}\bgr\}.
 \label{eq:chi}
\end{equation}
By \eqref{eq:budget}, $0\le\chi(I)\le\max\{0,\ln(\eta/(h\Lambda))\}\le\Lambda$.
The next lemma uses the truncated exponential
$E_k(s):=\sum_{i=0}^{k}s^{i}/i!$ and, for $\nu\ge1$, the parameter
$\alpha:=2\nu/\Lambda$.

\begin{lemma}
\label{lem:factorial}
Let $\nu\ge1$. For every interval $I$ of the construction and every
$0\le j\le\nu$,
\begin{equation}
 D_j(I)\le2\,\frac{(6\Lambda/\theta)^{j}}{j!}.
 \label{eq:factorial}
\end{equation}
\end{lemma}

\begin{proof}
If $I$ has no children, \eqref{eq:packing} below is trivial. Otherwise let
$I_1,\ldots,I_m$ be the children of $I$ and $\sigma_1<\cdots<\sigma_m$
their starting coefficients, so that $\mu(I_i)=\sigma_{i+1}-\sigma_i\le\sigma_{i+1}$
for $i<m$ and $\mu(I_m)\le\mu(I)$. The points
$u_i:=\ln(\theta\sigma_i/(2\Lambda))$, $2\le i\le m$, are nonnegative because
$\sigma_1\ge2\Lambda/\theta$, are at most $\chi(I)$, and are more than
$\theta/4$ apart by \eqref{eq:separation}. Hence, for every nondecreasing
$f\ge0$, comparing the disjoint intervals $[u_i,u_i+\theta/4]$ with an
integral gives
\begin{equation}
 \sum_{i=1}^{m}f(\chi(I_i))
 \le f(\chi(I))+\frac4\theta\int_0^{\chi(I)+\theta/4}f(u)\,du,
 \label{eq:riemann}
\end{equation}
where the first term accounts for the last child. Take $f(u)=E_k(\alpha u)$.
Since $\int_0^{v}E_k(\alpha u)\,du=(E_{k+1}(\alpha v)-1)/\alpha$ and
$E_{k+1}(u+v)\le e^{v}E_{k+1}(u)$ for $u,v\ge0$, the right side of
\eqref{eq:riemann} is at most
$(1+4e^{\alpha\theta/4}/(\alpha\theta))E_{k+1}(\alpha\chi(I))$, and
$\alpha\theta<4/(h\Lambda)<1/5$ gives
\begin{equation}
 \sum_{i=1}^{m}E_k\big(\alpha\chi(I_i)\big)
 \le\frac{6}{\alpha\theta}E_{k+1}\big(\alpha\chi(I)\big)
 \qquad(k\ge0).
 \label{eq:packing}
\end{equation}
By induction on $j$, simultaneously for all intervals, \eqref{eq:packing}
gives $D_j(I)\le(6/(\alpha\theta))^{j}E_j(\alpha\chi(I))$: the case $j=0$ is
trivial, and $D_{j+1}(I)=\sum_iD_j(I_i)$. Finally, $\chi(I)\le\Lambda$ and
$\alpha\Lambda=2\nu\ge2j$, so the terms of $E_j(\alpha\Lambda)$, read from
the top, decrease by factors of at least two, whence
$E_j(\alpha\chi(I))\le2(\alpha\Lambda)^{j}/j!$. Substituting proves
\eqref{eq:factorial}.
\end{proof}

It remains to bound the length of a root interval $J$. Every step of $J$
lies either in the prefix of a nonterminal interval or in a terminal
interval. Prefixes have at most $R$ steps and occur at depths
$0,\ldots,\nu-1$, and terminal intervals have at most
$\lceil4/\theta\rceil\le6/\theta$ steps and occur at depth $\nu$. Let
$\nu\ge1$ and $z:=6\Lambda/\theta$. Since $z/\nu>18$, the sum
$\sum_{j<\nu}z^{j}/j!$ is dominated by its last term, so
Lemma~\ref{lem:factorial} gives $\sum_{j<\nu}D_j(J)<3z^{\nu-1}/(\nu-1)!$ and
$D_\nu(J)\le2z^{\nu}/\nu!$. With $R\le z$ this yields
\begin{equation}
 |J|\le R\sum_{j=0}^{\nu-1}D_j(J)+\bgl\lceil\frac4\theta\bgr\rceil D_\nu(J)
 \le\bgl(3\nu+\frac{12}{\theta}\bgr)\frac{(6\Lambda/\theta)^{\nu}}{\nu!},
 \label{eq:per-root}
\end{equation}
and the same bound holds when $\nu=0$, where $J$ itself is terminal. Each
level of prefixes costs one factor of $\Lambda$ and the terminal intervals
cost none, so the degree in $\Lambda$ is $\nu$ and not $\nu+1$.

\subsection{Completing the proof}
\label{sec:completion}

The stepsize threshold $\eta_0$ in \eqref{eq:def-onset} is chosen so that
every $\eta\ge\eta_0$ satisfies
\begin{equation}
 \eta\ge4,\qquad \Lambda>10,\qquad 4\Lambda\le\eta\theta,\qquad
 C\Lambda^{p}<\widehat T
 \label{eq:onset-conditions}
\end{equation}
(Appendix~\ref{app:completion}). The first three conditions were used above,
and the last one is used at the end of the proof. We now add up the steps
before $U$. Besides step zero and the fewer than $4n\gamma^{-2}$ steps of
$\mathcal T_-$, there are the root intervals of Lemma~\ref{lem:refine}(i),
whose number is below $5\gamma^{-2}(n+h\ln(4\gamma^{-2}))$ by
\eqref{eq:budget}. Applying \eqref{eq:per-root} to each of them and using
$3\nu+12h\gamma^{-2}\le15h\gamma^{-2}$ gives
\begin{equation}
 U\le C\,\frac{(6h\gamma^{-2}\Lambda)^{\nu}}{\nu!},
 \qquad
 C=80\gamma^{-4}h\big(n+h\ln(4\gamma^{-2})\big).
 \label{eq:structural}
\end{equation}
The factorial in \eqref{eq:structural} is the saving of
Section~\ref{sec:packing}. It remains to absorb the fixed coefficient
$(6h\gamma^{-2})^{\nu}/\nu!$ into the exponent, which is how $p$ in
\eqref{eq:def-exponent} arises. If $\nu=0$, then $p=0$ and
\eqref{eq:structural} reads $U\le C$. If $\nu\ge1$, let $b:=\max\{2,\nu\}$.
From $\nu!\ge(\nu/e)^{\nu}$, $e<3$, and the definition of $p$,
\begin{equation}
 \frac{(6h\gamma^{-2})^{\nu}}{\nu!}
 \le\bgl(\frac{18h\gamma^{-2}}{\nu}\bgr)^{\nu}
 \le b^{\,p-\nu}.
 \label{eq:absorb}
\end{equation}
If $\Lambda\ge b$, then \eqref{eq:structural} and \eqref{eq:absorb} give
$U\le C\Lambda^{\nu}b^{p-\nu}\le C\Lambda^{p}$. If instead $10<\Lambda<b$,
then $b=\nu$ and $\Lambda<\nu$; here \eqref{eq:budget} and $\mu_t>1/2$ on
$\mathcal T_+$ give the crude count
$U=|\mathcal T_-|+|\mathcal T_+|<C+4\gamma^{-2}e^{\Lambda}\le C(1+e^{\Lambda})$,
and $e^{\Lambda}\le\Lambda^{\Lambda}\le\Lambda^{\nu}$ together with
$p\ge\nu+1$ gives $U\le2C\Lambda^{\nu}<C\Lambda^{p}$. The threshold $\eta_0$ therefore
never needs $\Lambda\ge\nu$, which would be an exponential stepsize
requirement. In every case $U\le C\Lambda^{p}$, and by
\eqref{eq:onset-conditions} the right side is smaller than
$\widehat T$. Since $U=\min\{\tau,\widehat T\}$, this forces
$U=\tau$, and $\tau\le C(\ln\eta)^{p}$ is \eqref{eq:main-bound}.
This completes the proof of Theorem~\ref{thm:main}(i) for
$0<\gamma<1/\sqrt2$.

\paragraph{Remark.}
Since $U=\tau$, \eqref{eq:structural} itself gives the bound
$\tau\le C(6h\gamma^{-2}\ln\eta)^{\nu}/\nu!$ on the phase transition time,
which is at least as sharp as \eqref{eq:main-bound} whenever
$\ln\eta\ge\max\{2,\nu\}$ (Proposition~\ref{prop:b4-factorial}). For
$\nu\ge1$ it has the smaller degree $\nu$, but its coefficient
$(6h\gamma^{-2})^{\nu}/\nu!$ can be superpolynomial in $1/\gamma$. The
exponent $p$ absorbs this coefficient and keeps the prefactor $C$ polynomial
in $n$ and $1/\gamma$.

\section{Conclusion}
\label{sec:conclusion}

We consider gradient descent (GD) with a constant stepsize for logistic
regression on linearly separable data. We show that GD with a large stepsize
$\eta=1/\epsilon$ achieves an $O(\ln^{p}(1/\epsilon))$ step complexity, where
$p$ depends only on the margin and the rank of the data. This improves over
the prior rate in arbitrary dimension for fixed problem dependent parameters.
We leave it to future work to tighten the constants in our bounds and to
establish new lower bounds.


\subsection*{AI use statement}




In this work, we used generative AI tools to assist with drafting and polishing the writing, developing parts of the proof strategy, and reviewing the proofs. The finite-horizon argument and the tightened interval-counting argument were developed through iterative collaboration with GPT-5.6 Sol. GPT-5.6 Sol was also used to simplify the proof and improve its organization and presentation.
We have reviewed all AI-assisted work and take responsibility for the final content of this work, including text, claims or artifacts produced with the aid of generative AI.

\bibliography{iclr2027_conference}
\bibliographystyle{iclr2027_conference}

\appendix

\section{Proofs for Section~\ref{sec:main}}
\label{app:constants}

This appendix states the general form of Theorem~\ref{thm:main}(i), in which
an integer parameter $h\ge2$ sets the tolerance $\theta_h=\gamma^{2}/h$ of the
interval construction (Theorem~\ref{thm:general}, proved in
Appendix~\ref{app:proofs}), proves the parameter dependence stated in
Section~\ref{sec:main} (Proposition~\ref{prop:parameters}) and the
large-margin case (Proposition~\ref{prop:large-margin-app}), recalls the
stable phase, and proves Corollary~\ref{cor:accuracy}. Throughout,
Assumption~\ref{ass:data} holds, GD is run from $\mathbf{w}_0=\mathbf{0}$, and $h\ge2$ is an integer; the
restriction $\gamma<1/\sqrt2$ is stated
wherever it is needed.

\subsection{Explicit constants and the general theorem}
\label{app:general}

Set
\begin{equation}
 \theta_h:=\frac{\gamma^2}{h},
 \qquad
 \nu_h:=\max\bigg\{0,\ \min\bigg\{r,\ \bigg\lfloor\gamma^{-2}+\frac{\gamma^{-2}-1}{h-1}\bigg\rfloor\bigg\}-2\bigg\}.
 \label{eq:depth}
\end{equation}
Since $h\ge2$, the difference $\gamma^{2}-\theta_h=\gamma^{2}(1-1/h)$ is
positive, and multiplying numerator and denominator by $h/\gamma^{2}$ and
then writing $h/(h-1)=1+1/(h-1)$ gives, for every integer $h\ge2$,
\begin{equation}
 \frac{1-\theta_h}{\gamma^{2}-\theta_h}
 =\frac{h-\gamma^{2}}{\gamma^{2}(h-1)}
 =\gamma^{-2}+\frac{\gamma^{-2}-1}{h-1}.
 \label{eq:floor-identity}
\end{equation}
The quantity
$\nu_h$ bounds the depth of the recursive interval
construction. Increasing $h$ lowers $\theta_h$ and does not increase the
depth bound $\nu_h$, but it increases the bound on the number of children
created at each refinement. Since
$\gamma\le1$ and $h\ge2$, we have
$0\le(\gamma^{-2}-1)/(h-1)\le\gamma^{-2}-1$, so the quantity inside the
floor in \eqref{eq:depth} is at most $2\gamma^{-2}-1$, and therefore
\begin{equation}
 0\le\nu_h<2\gamma^{-2}.
 \label{eq:depth-range}
\end{equation}
Moreover, $\nu_h=0$ whenever $r\le2$. For $\nu_h\ge1$ define the exponent
\begin{equation}
 p_h:=\nu_h+\bigg\lceil\frac{\nu_h\ln(18h\gamma^{-2}/\nu_h)}{\ln(\max\{2,\nu_h\})}\bigg\rceil,
 \label{eq:exponent}
\end{equation}
and set $p_h:=0$ when $\nu_h=0$. When $\nu_h\ge1$, \eqref{eq:depth-range}
gives $18h\gamma^{-2}/\nu_h>9h\ge18$, and
$\ln(\max\{2,\nu_h\})\ge\ln2>0$, so the quotient in \eqref{eq:exponent} is
positive and $p_h$ is an integer with $p_h\ge\nu_h+1$. Finally, put
\begin{align}
 C_h&:=80\gamma^{-4}h\big(n+h\ln(4\gamma^{-2})\big),
 \label{eq:prefactor}\\
 K_h&:=\exp(4)+h\gamma^{-2}+p_h+\ln(2\gamma^{-2}C_h),
 \qquad
 \eta_h^0:=\big\lceil128\,K_h(1+\ln K_h)\big\rceil.
 \label{eq:onset}
\end{align}

\begin{theorem}[General form of Theorem~\ref{thm:main}(i)]
\label{thm:general}
Suppose Assumption~\ref{ass:data} holds with $0<\gamma<1/\sqrt2$. For every
integer $h\ge2$ and every $\eta\ge\eta_h^0$,
\begin{equation}
 \tau\le C_h\,(\ln\eta)^{p_h}.
 \label{eq:general-bound}
\end{equation}
\end{theorem}

Theorem~\ref{thm:general} is proved in Appendix~\ref{app:proofs}. The
low-margin case of Theorem~\ref{thm:main}(i) is Theorem~\ref{thm:general}
instantiated at
\begin{equation}
 h_*:=\lceil\ln(16\gamma^{-2})\rceil,
 \qquad
 C:=C_{h_*},
 \qquad
 p:=p_{h_*},
 \qquad
 \eta_0:=\eta_{h_*}^0.
 \label{eq:selected}
\end{equation}
For $\gamma<1/\sqrt2$ we have $\gamma^{-2}>2$, hence
$\ln(16\gamma^{-2})>\ln32>3$ and $h_*\ge4$. The instantiation is therefore
admissible, the definitions \eqref{eq:def-depth}--\eqref{eq:def-onset}
of the main text are \eqref{eq:depth}--\eqref{eq:onset} at $h=h_*$, and
\eqref{eq:general-bound} at $h=h_*$ is \eqref{eq:main-bound}.

\subsection{Parameter dependence}
\label{app:parameters}

The next proposition gives the orders of the selected quantities in $n$, $r$,
and $\gamma$, with absolute constants in $O(\cdot)$; the bounds \eqref{eq:param-C}--\eqref{eq:param-onset} hold on the whole
range $0<\gamma<1/\sqrt2$, and only the statements after
\eqref{eq:rank-condition} are asymptotic.

\begin{proposition}[Parameter dependence]
\label{prop:parameters}
Let $0<\gamma<1/\sqrt2$, let $h_*$, $C$, $p$, and $\eta_0$ be as in
\eqref{eq:selected}, and write $\nu:=\nu_{h_*}$. Then
\begin{align}
 C&=O\big(n\gamma^{-4}\ln(\gamma^{-2})+\gamma^{-4}\ln^{3}(\gamma^{-2})\big),
 \label{eq:param-C}\\
 p&\le\nu\bigg(2+\frac{\ln(18h_*\gamma^{-2})}{\ln2}\bigg)
 =O\big(\min\{r,\gamma^{-2}\}\ln(\gamma^{-2})\big),
 \label{eq:param-p}\\
 \eta_0&=O(A\ln A),
 \qquad\text{where}\ \ A:=\gamma^{-2}\ln(\gamma^{-2})+\ln(en).
 \label{eq:param-onset}
\end{align}
In particular $p=O(\gamma^{-2}\ln(\gamma^{-2}))$, and $p=0$ whenever
$r\le2$. If moreover the rank satisfies
\begin{equation}
 r\ge\bigg\lfloor\frac{h_*-\gamma^{2}}{\gamma^{2}(h_*-1)}\bigg\rfloor
 =\bigg\lfloor\gamma^{-2}+\frac{\gamma^{-2}-1}{h_*-1}\bigg\rfloor,
 \label{eq:rank-condition}
\end{equation}
then, as $\gamma\to0$,
\begin{equation}
 p=\gamma^{-2}+O\bigg(\gamma^{-2}\frac{\ln\ln(\gamma^{-2})}{\ln(\gamma^{-2})}\bigg)
 =(1+o(1))\gamma^{-2},
 \label{eq:p-asymptotic}
\end{equation}
and the threshold on the right of \eqref{eq:rank-condition} is itself
$(1+o(1))\gamma^{-2}$.
\end{proposition}

\begin{proof}
Put $\omega:=\ln(\gamma^{-2})$. Since $\gamma^{-2}>2$, we have
$\ln2<\omega<\gamma^{-2}$; in particular every absolute constant is
$O(\omega)$.

\emph{The parameter $h_*$.} Since $\ln(16\gamma^{-2})=\omega+4\ln2$ and a
ceiling adds less than one,
\begin{equation}
 \omega<h_*<\omega+4\ln2+1,
 \label{eq:param-h}
\end{equation}
so $h_*=O(\omega)$.

\emph{The prefactor $C$.} Since $\ln(4\gamma^{-2})=\omega+2\ln2=O(\omega)$ and
$h_*=O(\omega)$, \eqref{eq:prefactor} at $h=h_*$ gives
$C=O(\gamma^{-4}\omega(n+\omega^{2}))$, which is \eqref{eq:param-C}.

\emph{The exponent $p$.} If $r\le2$, then $\nu=0$, as observed after
\eqref{eq:depth-range}, and $p=0$ by definition. More generally,
\eqref{eq:param-p} is trivial when $\nu=0$, so let $\nu\ge1$. The ceiling in
\eqref{eq:exponent} adds less than one, the logarithm in
\eqref{eq:exponent} is positive with $\ln(\max\{2,\nu\})\ge\ln2$ below it,
and $18h_*\gamma^{-2}/\nu\le18h_*\gamma^{-2}$, so that, using $1\le\nu$ in
the last step,
\[
 p<\nu+1+\frac{\nu\ln(18h_*\gamma^{-2})}{\ln2}
 \le\nu\bigg(2+\frac{\ln(18h_*\gamma^{-2})}{\ln2}\bigg).
\]
Here $\ln(18h_*\gamma^{-2})=\ln18+\ln h_*+\omega=O(\omega)$ by
\eqref{eq:param-h}, and
$\nu\le\min\{r,2\gamma^{-2}\}\le2\min\{r,\gamma^{-2}\}$, because
$\nu\le r-2$ by \eqref{eq:depth} when $\nu\ge1$ and $\nu<2\gamma^{-2}$ by
\eqref{eq:depth-range}. This proves \eqref{eq:param-p}, and
$p=O(\gamma^{-2}\omega)$ follows.

\emph{The stepsize threshold $\eta_0$.} Each of the four terms of $K_{h_*}$ in
\eqref{eq:onset} is $O(A)$: $\exp(4)=O(1)$,
$h_*\gamma^{-2}=O(\gamma^{-2}\omega)$ by \eqref{eq:param-h},
$p=O(\gamma^{-2}\omega)$ was just proved, and
$\ln(2\gamma^{-2}C)=O(\omega+\ln(en))$ by \eqref{eq:param-C}, since
$\ln\omega<\omega$ and
$\ln(n+\omega^{2})\le\ln n+\ln(1+\omega^{2})\le\ln n+2\omega$; here
$\omega<\gamma^{-2}\omega\le A$ and $A\ge\gamma^{-2}\omega+1>2$, because
$\gamma^{-2}\omega>2\ln2>1$ and $\ln(en)\ge1$. Hence $K_{h_*}\le cA$ for an
absolute constant $c\ge1$. On the other hand $K_{h_*}\ge\exp(4)>1$, because the other
three terms of \eqref{eq:onset} are nonnegative, the last one since
$C\ge80\gamma^{-4}h_*n>1$. The function $s\mapsto s(1+\ln s)$ is increasing on
$[1,\infty)$, and the ceiling in \eqref{eq:onset} adds less than one, so
\[
 \eta_0<128\,K_{h_*}(1+\ln K_{h_*})+1
 \le128\,cA\,(1+\ln c+\ln A)+1
 =O(A\ln A),
\]
because $A>2$ gives $\ln A>\ln2$, so that $1+\ln c=O(\ln A)$ and
$1=O(A\ln A)$. This proves \eqref{eq:param-onset}.

\emph{The asymptotic exponent.} All $O(\cdot)$ and $o(\cdot)$ statements
below refer to $\gamma\to0$, so that $\omega\to\infty$. Let
$x:=\gamma^{-2}+(\gamma^{-2}-1)/(h_*-1)$; the two expressions in
\eqref{eq:rank-condition} agree by \eqref{eq:floor-identity} at $h=h_*$.
Since $0<\gamma^{-2}-1<\gamma^{-2}$ and $h_*>\omega$ by \eqref{eq:param-h},
\begin{equation}
 x=\gamma^{-2}\big(1+O(1/\omega)\big),
 \label{eq:param-x}
\end{equation}
so the threshold $\lfloor x\rfloor$ in \eqref{eq:rank-condition} is
$(1+o(1))\gamma^{-2}$. Assume now \eqref{eq:rank-condition}. Since
$x>\gamma^{-2}>2$, the minimum in \eqref{eq:depth} equals
$\lfloor x\rfloor\ge2$, so $\nu=\lfloor x\rfloor-2$, and \eqref{eq:param-x}
gives
\begin{equation}
 \nu=\gamma^{-2}\big(1+O(1/\omega)\big).
 \label{eq:param-nu}
\end{equation}
In particular $\nu\to\infty$, so eventually $\max\{2,\nu\}=\nu$ and
$\ln\nu=(1+o(1))\omega$, while $18h_*\gamma^{-2}/\nu=O(h_*)=O(\omega)$ by
\eqref{eq:param-h}, so $\ln(18h_*\gamma^{-2}/\nu)=O(\ln\omega)$. Since the
ceiling in \eqref{eq:exponent} adds less than one,
\[
 1\le p-\nu
 <1+\frac{\nu\ln(18h_*\gamma^{-2}/\nu)}{\ln\nu}
 =O\bigg(\gamma^{-2}\frac{\ln\omega}{\omega}\bigg).
\]
Together with \eqref{eq:param-nu}, and since $1/\omega=O(\ln\omega/\omega)$,
this gives $p-\gamma^{-2}=(p-\nu)+(\nu-\gamma^{-2})=O(\gamma^{-2}\ln\omega/\omega)$,
which is the first equality in \eqref{eq:p-asymptotic}. The second follows
because $\ln\omega/\omega\to0$.
\end{proof}

Under the rank condition \eqref{eq:rank-condition}, as $\gamma\to0$, the
proof separates two contributions to $p-\gamma^{-2}$. The depth error
$\nu-\gamma^{-2}$ is $O(\gamma^{-2}/\ln(\gamma^{-2}))$ by
\eqref{eq:param-nu}, while the overhead $p-\nu$, which absorbs the fixed
coefficient of the factorial bound into the exponent, carries the
additional factor $\ln\ln(\gamma^{-2})$ and determines the order displayed
in \eqref{eq:p-asymptotic}. The orders stated in
Section~\ref{sec:main} are those of Proposition~\ref{prop:parameters}. The
bounds on $C$ and $p$ are \eqref{eq:param-C} and \eqref{eq:param-p}, and
$A\le(\gamma^{-2}+\ln n)(1+\ln(\gamma^{-2}))$ turns \eqref{eq:param-onset}
into $\eta_0=\widetilde O(\gamma^{-2}+\ln n)$, where $\widetilde O$ hides
factors polylogarithmic in $\gamma^{-2}+\ln n$.

\subsection{The large-margin case}
\label{app:large-margin}

This part proves Proposition~\ref{prop:large-margin-app}, the second case of
Theorem~\ref{thm:main}(i), and shows along the way that for
$\gamma\ge1/\sqrt2$ the loss decreases monotonically from initialization for
every stepsize. Neither the parameter $h$ nor the horizon \eqref{eq:horizon} is
involved. We write $\ell(u):=\ln(1+e^{-u})$, so that
$L(\mathbf{w})=\frac1n\sum_{i=1}^{n}\ell(\ip{\mathbf{z}_i}{\mathbf{w}})$.

\begin{proposition}[Large margin; the second case of Theorem~\ref{thm:main}(i)]
\label{prop:large-margin-app}
Suppose Assumption~\ref{ass:data} holds with $\gamma\ge1/\sqrt2$. If
$\eta\ge64n^2$, then $\tau\le1$.
\end{proposition}


\begin{proof}[Proof of Proposition~\ref{prop:large-margin-app}]
For each $i$, set $a_i:=\ip{\mathbf{z}_i}{\wstar}\ge\gamma$ and
$\mathbf{u}_i:=\mathbf{z}_i-a_i\wstar$, so that $\ip{\mathbf{u}_i}{\wstar}=0$
and $a_i^2+\norm{\mathbf{u}_i}^2=\norm{\mathbf{z}_i}^2\le1$. Hence
$\norm{\mathbf{u}_i}^2\le1-\gamma^2$ and $\norm{\mathbf{z}_i}^2\ge\gamma^2$,
and the Cauchy--Schwarz inequality gives, for all $i,j$,
\begin{equation}
 \ip{\mathbf{z}_i}{\mathbf{z}_j}
 =a_ia_j+\ip{\mathbf{u}_i}{\mathbf{u}_j}
 \ge\gamma^2-(1-\gamma^2)
 =2\gamma^2-1
 \ge0,
 \label{eq:b5-acute}
\end{equation}
where the last inequality is the hypothesis $\gamma^2\ge1/2$.

By the first identity in \eqref{eq:update}, which holds for every $t\ge0$
and every $\eta>0$, and by \eqref{eq:b5-acute},
\begin{equation}
 \ip{\mathbf{z}_i}{\mathbf{w}_{t+1}}-\ip{\mathbf{z}_i}{\mathbf{w}_t}
 =\frac{\eta}{n}\sum_{j=1}^{n}q_{j,t}\ip{\mathbf{z}_i}{\mathbf{z}_j}
 \ge0
 \qquad(i=1,\ldots,n;\ t\ge0;\ \eta>0).
 \label{eq:b5-monotone}
\end{equation}
Thus every signed margin is nondecreasing in $t$, and since $\ell$ is
decreasing, $L(\mathbf{w}_{t+1})\le L(\mathbf{w}_t)$ for every $t\ge0$ and
every $\eta>0$, which is the assertion made after Theorem~\ref{thm:main}.

It remains to examine the first step. Since $\mathbf{w}_0=\mathbf{0}$, every
weight equals $q_{j,0}=1/2$, so
$\mathbf{w}_1=\frac{\eta}{2n}\sum_{j=1}^{n}\mathbf{z}_j$. Keeping only the
term $j=i$ in the sum below, which is allowed by \eqref{eq:b5-acute},
\[
 \ip{\mathbf{z}_i}{\mathbf{w}_1}
 =\frac{\eta}{2n}\sum_{j=1}^{n}\ip{\mathbf{z}_i}{\mathbf{z}_j}
 \ge\frac{\eta}{2n}\norm{\mathbf{z}_i}^2
 \ge\frac{\eta}{4n}
 \qquad(i=1,\ldots,n).
\]
Finally, $\ell(u)\le e^{-u}$ and $e^{s}\ge s^2/2$ for $s>0$, so for
$\eta\ge64n^2$,
\[
 L(\mathbf{w}_1)
 \le\frac1n\sum_{i=1}^{n}\exp\big(-\ip{\mathbf{z}_i}{\mathbf{w}_1}\big)
 \le\exp\bigg(-\frac{\eta}{4n}\bigg)
 \le\frac{32n^2}{\eta^2}
 \le\frac{1}{2\eta}
 <\frac1\eta ,
\]
where the next-to-last inequality is equivalent to $\eta\ge64n^2$. By the
definition \eqref{eq:tau} of the phase transition time, $\tau\le1$.
\end{proof}

\subsection{The stable phase}
\label{app:stable}

Part~(ii) of Theorem~\ref{thm:main} is the monotonicity statement in the
stable-phase part of \citet[Theorem~1]{wu2024large}, which we recall. If
$L(\mathbf{w}_s)\le1/\eta$ at some step $s$, then
$L(\mathbf{w}_{t+1})\le L(\mathbf{w}_t)$ for every $t\ge s$, and for every
$t>s$,
\begin{equation}
 L(\mathbf{w}_t)\le
 \frac{2F(\mathbf{w}_s)+\ln^{2}\big(\gamma^{2}\eta(t-s)\big)}{\gamma^{2}\eta(t-s)},
 \qquad\text{where}\ \ F(\mathbf{w}):=\frac1n\sum_{i=1}^n\exp(-y_i\mathbf{x}_i^\top\mathbf{w});
 \label{eq:stable-rate}
\end{equation}
the rate is their Lemma~11, and the monotonicity part holds already under
the weaker condition $L(\mathbf{w}_s)\le2/\eta$ by their Lemma~10. Since
$L(\mathbf{w}_{\tau})\le1/\eta$ by \eqref{eq:tau}, taking $s=\tau$
gives part~(ii). With $t=\tau+k$, the rate \eqref{eq:stable-rate} also
shows how the loss continues to decrease after the phase transition at a
fixed stepsize. Here the term $F(\mathbf{w}_{\tau})$ is controlled by
the definition of $\tau$. Writing $a_i:=\exp(-y_i\mathbf{x}_i^\top\mathbf{w})$,
we have $1+nF(\mathbf{w})=1+\sum_{i=1}^na_i\le\prod_{i=1}^n(1+a_i)=\exp(nL(\mathbf{w}))$,
so $F(\mathbf{w}_{\tau})\le(\exp(n/\eta)-1)/n$, which is at most
$(e-1)/\eta$ whenever $\eta\ge n$.

\subsection{Proof of Corollary~\ref{cor:accuracy}}
\label{app:corollaries}

\begin{proof}[Proof of Corollary~\ref{cor:accuracy}]
(i) Let $0<\epsilon\le\epsilon_0$ and $\eta=1/\epsilon$, so that
$\eta\ge\eta_0$ if $\gamma<1/\sqrt2$ and $\eta\ge64n^2$ if
$\gamma\ge1/\sqrt2$. By Theorem~\ref{thm:main}(i), with $C=1$ and $p=0$ in
the latter case,
\[
 \tau\le C\,(\ln\eta)^{p}=C\,\ln^{p}(1/\epsilon).
\]
By the definition \eqref{eq:tau} of the phase transition time,
$L(\mathbf{w}_t)\le1/\eta=\epsilon$ at $t=\tau$, and by
Theorem~\ref{thm:main}(ii) the loss decreases monotonically from step
$\tau$ onwards.

(ii) Let $p\ge1$, so that $\gamma<1/\sqrt2$, and recall
$T_0=\lceil C(\ln\eta_0)^{p}\rceil$. For every integer $T\ge T_0$, the
stepsize $\eta=\exp((T/C)^{1/p})$ satisfies $\ln\eta=(T/C)^{1/p}\ge\ln\eta_0$,
hence $\eta\ge\eta_0$, and Theorem~\ref{thm:main}(i) gives
\[
 \tau\le C\,(\ln\eta)^{p}=C\cdot\frac TC=T.
\]
Therefore $L(\mathbf{w}_{\tau})\le1/\eta$, and by Theorem~\ref{thm:main}(ii)
the loss decreases monotonically from step $\tau\le T$ onwards, so
$L(\mathbf{w}_T)\le1/\eta=\exp(-(T/C)^{1/p})$.
If instead $p=0$, let $0<\epsilon\le\epsilon_0$. By part~(i),
$L(\mathbf{w}_t)\le\epsilon$ for some $t\le C$ and the loss decreases
monotonically afterwards; hence $L(\mathbf{w}_T)\le\epsilon$ for every integer $T\ge C$.
\end{proof}

\section{Proofs for Section~\ref{sec:analysis}}
\label{app:proofs}

This appendix proves Theorem~\ref{thm:general} in full, following the
structure of Section~\ref{sec:analysis}.
Throughout, Assumption~\ref{ass:data} holds and GD is run from
$\mathbf{w}_0=\mathbf{0}$; in
Appendices~\ref{app:trajectory}--\ref{app:completion} we fix an integer
$h\ge2$ and use the general quantities $\theta_h$, $\nu_h$, $p_h$, $C_h$,
$K_h$, and $\eta_h^0$ of \eqref{eq:depth}--\eqref{eq:onset}. The
quantities of Section~\ref{sec:analysis} are the case $h=h_*$ of
\eqref{eq:selected}, that is, $\theta=\theta_{h_*}$, $\nu=\nu_{h_*}$,
$p=p_{h_*}$, $C=C_{h_*}$, and $\eta_0=\eta_{h_*}^0$.
Appendices~\ref{app:trajectory}--\ref{app:packing} use only three
conditions on the stepsize, namely $\eta\ge4$, $\Lambda=\ln\eta>10$, and
the root condition $4\Lambda\le\eta\theta_h$. Lemma~\ref{lem:b4-onset} in
Appendix~\ref{app:completion} shows that every $\eta\ge\eta_h^0$ satisfies
these conditions. The phase transition
time $\tau$ is finite for every $\eta>0$ by the bound
\eqref{eq:wu-tau} of \citet{wu2024large}, so $U$ in \eqref{eq:horizon} is
well defined. All other notation is as in Section~\ref{sec:analysis}.

\subsection{Trajectory estimates}
\label{app:trajectory}
This part proves the update decomposition \eqref{eq:update}, the bound on
the cumulative coefficient on the finite horizon, Lemma~\ref{lem:basic}, and
one running-average estimate that gives \eqref{eq:path} and is used again in
Appendix~\ref{app:refinement}. Nothing in it depends on $h$. The standing
assumption is $\eta\ge4$. We write $\ell(u):=\ln(1+e^{-u})$, so that
$L(\mathbf{w})=\frac1n\sum_{i=1}^{n}\ell(\ip{\mathbf{z}_i}{\mathbf{w}})$ by
\eqref{eq:loss}. Since $\ell'(u)=-1/(1+e^{u})<0$ and
$\ell''(u)=e^{u}/(1+e^{u})^{2}>0$, the function $\ell$ is decreasing and
convex, $L$ is convex, and $\ell(u)\le e^{-u}$ because $\ln(1+x)\le x$
for $x\ge0$. The weights in \eqref{eq:weights} satisfy $0<q_{i,t}<1$, so
that $0<Q_t<n$, $\mathbf{v}_t$ is well defined, and $0<\mu_t<\eta$. By
Assumption~\ref{ass:data}, $\gamma\le1$, and every vector
$\mathbf{b}=\sum_ic_i\mathbf{z}_i$ of the convex hull of the signed
examples, where $c_i\ge0$ and $\sum_ic_i=1$, satisfies
\begin{equation}
 \norm{\mathbf{b}}\le\sum_{i=1}^{n}c_i\norm{\mathbf{z}_i}\le1,
 \qquad
 \ip{\mathbf{b}}{\wstar}=\sum_{i=1}^{n}c_i\ip{\mathbf{z}_i}{\wstar}\ge\gamma.
 \label{eq:b1-hull}
\end{equation}

\begin{lemma}[Update decomposition]
\label{lem:b1-update}
For every $t\ge0$,
\begin{equation}
 \nabla L(\mathbf{w}_t)
 =-\frac{Q_t}{n}\mathbf{v}_t
 =-\frac{\mu_t}{\eta}\mathbf{v}_t,
 \label{eq:b1-gradient}
\end{equation}
and consequently \eqref{eq:update} holds.
\end{lemma}

\begin{proof}
Since $\ell'(\ip{\mathbf{z}_i}{\mathbf{w}_t})=-q_{i,t}$, the chain rule
gives
$\nabla L(\mathbf{w}_t)=-\frac1n\sum_{i}q_{i,t}\mathbf{z}_i=-(Q_t/n)\mathbf{v}_t$,
and $Q_t/n=\mu_t/\eta$. Substituting \eqref{eq:b1-gradient} into
\eqref{eq:gd} with $\eta_t\equiv\eta$ gives
$\mathbf{w}_{t+1}-\mathbf{w}_t=\mu_t\mathbf{v}_t$. The coefficients
$q_{i,t}/Q_t$ are positive and sum to one, so $\mathbf{v}_t$ lies in the
convex hull of the signed examples, and \eqref{eq:b1-hull} gives the last
two properties in \eqref{eq:update}.
\end{proof}

Recall $\widehat T=\lfloor\gamma^{2}\exp(\eta/4)/\eta\rfloor$ and
$U=\min\{\tau,\widehat T\}$ from \eqref{eq:horizon}. $U$ is a
well-defined nonnegative integer because $\tau$ is finite, as recalled
at the start of this appendix. All counting in this appendix is restricted to the
steps $0,1,\ldots,U-1$, and the horizon is removed only in the last step of
the proof of Theorem~\ref{thm:general}, because the following bound is
proved only up to $\widehat T$. Its proof applies the split optimization
technique introduced by \citet[Lemma~7]{wu2024large}.

\begin{lemma}[Norm and cumulative-coefficient bounds]
\label{lem:b1-budget}
Let $\eta\ge4$. For every integer $1\le T\le\widehat T$,
\begin{equation}
 \norm{\mathbf{w}_T}\le\frac{2\eta}{\gamma},
 \qquad
 \sum_{t=0}^{T-1}\mu_t\le2\eta\gamma^{-2}.
 \label{eq:b1-horizon-budget}
\end{equation}
\end{lemma}

\begin{proof}
Write $\mathbf{g}_t:=\nabla L(\mathbf{w}_t)$. By \eqref{eq:b1-gradient} and
\eqref{eq:update},
\[
 \norm{\mathbf{g}_t}=\frac{\mu_t}{\eta}\norm{\mathbf{v}_t}\le\frac{\mu_t}{\eta},
 \qquad
 \ip{\mathbf{g}_t}{\wstar}=-\frac{\mu_t}{\eta}\ip{\mathbf{v}_t}{\wstar}
 \le-\frac{\gamma\mu_t}{\eta}.
\]
Define the fixed comparators
\[
 \mathbf{u}_1:=\frac{\eta}{4\gamma}\wstar,
 \qquad
 \mathbf{u}_2:=\frac{\eta}{2\gamma}\wstar,
 \qquad
 \mathbf{u}:=\mathbf{u}_1+\mathbf{u}_2=\frac{3\eta}{4\gamma}\wstar.
\]
The component $\mathbf{u}_2$ cancels the quadratic term of a step. Indeed,
the two bounds on $\mathbf{g}_t$ and $0<\mu_t<\eta$ give
\begin{equation}
 2\ip{\mathbf{g}_t}{\mathbf{u}_2}+\eta\norm{\mathbf{g}_t}^{2}
 \le-\mu_t+\frac{\mu_t^{2}}{\eta}
 =-\mu_t\bigg(1-\frac{\mu_t}{\eta}\bigg)
 \le0.
 \label{eq:b1-cancel}
\end{equation}
Since $\mathbf{w}_{t+1}-\mathbf{u}=(\mathbf{w}_t-\mathbf{u})-\eta\mathbf{g}_t$,
expanding the square and splitting $\mathbf{u}=\mathbf{u}_1+\mathbf{u}_2$
gives
\begin{align*}
 \norm{\mathbf{w}_{t+1}-\mathbf{u}}^{2}
 &=\norm{\mathbf{w}_t-\mathbf{u}}^{2}
 +2\eta\ip{\mathbf{g}_t}{\mathbf{u}_1-\mathbf{w}_t}
 +\eta\big(2\ip{\mathbf{g}_t}{\mathbf{u}_2}+\eta\norm{\mathbf{g}_t}^{2}\big)\\
 &\le\norm{\mathbf{w}_t-\mathbf{u}}^{2}
 +2\eta\big(L(\mathbf{u}_1)-L(\mathbf{w}_t)\big),
\end{align*}
where the inequality uses \eqref{eq:b1-cancel} and the gradient inequality
$L(\mathbf{u}_1)\ge L(\mathbf{w}_t)+\ip{\mathbf{g}_t}{\mathbf{u}_1-\mathbf{w}_t}$
of the convex function $L$. Every signed margin of $\mathbf{u}_1$ satisfies
$\ip{\mathbf{z}_i}{\mathbf{u}_1}=(\eta/(4\gamma))\ip{\mathbf{z}_i}{\wstar}\ge\eta/4$,
so $L(\mathbf{u}_1)\le\ell(\eta/4)\le e^{-\eta/4}$ because $\ell$ is
decreasing and $\ell(u)\le e^{-u}$. Summing over $t=0,\ldots,T-1$, using
$\mathbf{w}_0=\mathbf{0}$, and discarding the nonpositive term
$-2\eta\sum_{t<T}L(\mathbf{w}_t)$, we obtain
\[
 \norm{\mathbf{w}_T-\mathbf{u}}^{2}
 \le\norm{\mathbf{u}}^{2}+2\eta Te^{-\eta/4}
 \le\norm{\mathbf{u}}^{2}+2\gamma^{2},
\]
where the last inequality is
$T\le\widehat T\le\gamma^{2}e^{\eta/4}/\eta$. The triangle inequality
and $\sqrt{a+b}\le\sqrt a+\sqrt b$ for $a,b\ge0$ now give
\[
 \norm{\mathbf{w}_T}
 \le\norm{\mathbf{u}}+\norm{\mathbf{w}_T-\mathbf{u}}
 \le2\norm{\mathbf{u}}+\sqrt2\,\gamma
 =\frac{3\eta}{2\gamma}+\sqrt2\,\gamma
 \le\frac{2\eta}{\gamma}.
\]
The last inequality is equivalent to $2\sqrt2\,\gamma^{2}\le\eta$, which
holds because $\gamma\le1$ and $\eta\ge4>2\sqrt2$. For the second bound in
\eqref{eq:b1-horizon-budget}, Lemma~\ref{lem:b1-update} and
$\mathbf{w}_0=\mathbf{0}$ give
$\mathbf{w}_T=\sum_{t<T}\mu_t\mathbf{v}_t$, and since
$\ip{\mathbf{v}_t}{\wstar}\ge\gamma$ and $\mu_t>0$,
\[
 \gamma\sum_{t=0}^{T-1}\mu_t
 \le\bigg\langle\sum_{t=0}^{T-1}\mu_t\mathbf{v}_t,\wstar\bigg\rangle
 =\ip{\mathbf{w}_T}{\wstar}
 \le\norm{\mathbf{w}_T}
 \le\frac{2\eta}{\gamma}.
\]
Dividing by $\gamma$ completes the proof.
\end{proof}

Recall from Section~\ref{sec:geometry} the sets
$\mathcal T_-=\{t<U:\ \ip{\mathbf{z}_i}{\mathbf{w}_t}<0\text{ for some }i\}$
and $\mathcal T_+=\{0,1,\ldots,U-1\}\setminus\mathcal T_-$, so that every
signed margin is nonnegative at each $t\in\mathcal T_+$.

\begin{lemma}[General form of Lemma~\ref{lem:basic}]
\label{lem:b1-basic}
Let $\eta\ge4$. Then
\begin{equation}
 \sum_{t<U}\mu_t\le2\eta\gamma^{-2},
 \qquad
 |\mathcal T_-|<4n\gamma^{-2},
 \label{eq:b1-budget}
\end{equation}
and at every $t\in\mathcal T_+$,
\begin{equation}
 \mu_t\ge\frac{\eta}{2}L(\mathbf{w}_t)>\frac12,
 \qquad
 0\le\ip{\mathbf{w}_t}{\mathbf{v}_t}<\Lambda.
 \label{eq:b1-pointwise}
\end{equation}
\end{lemma}

Lemma~\ref{lem:basic} is this statement; neither it nor its proof involves
$h$.

\begin{proof}
The proof of Lemma~\ref{lem:basic} in Section~\ref{sec:geometry} is
complete except for the first bound in \eqref{eq:b1-budget}, which it defers
to this appendix. That bound is trivial when $U=0$, and for $U\ge1$ it is
Lemma~\ref{lem:b1-budget} with $T=U\le\widehat T$. The remaining
assertions then follow verbatim from the proof in
Section~\ref{sec:geometry}, which uses only the weights \eqref{eq:weights},
the convexity of $\ell$, and $L(\mathbf{w}_t)>1/\eta$ for
$t<U\le\tau$, by the definition \eqref{eq:tau} of $\tau$.
\end{proof}

The next lemma gives, for every starting step, the properties of the
running averages of update directions used in
Appendix~\ref{app:refinement}.

\begin{lemma}[Running averages of the update directions]
\label{lem:b1-average}
Let $\eta>0$, let $a\ge0$ be an integer and, for every integer $t>a$, define
\[
 S^{(a)}_t:=\sum_{k=a}^{t-1}\mu_k,
 \qquad
 \bar{\mathbf{v}}^{(a)}_t
 :=\frac{1}{S^{(a)}_t}\sum_{k=a}^{t-1}\mu_k\mathbf{v}_k
 =\frac{\mathbf{w}_t-\mathbf{w}_a}{S^{(a)}_t}.
\]
Then, for every $t>a$,
\begin{equation}
\begin{gathered}
 \bar{\mathbf{v}}^{(a)}_t\in\conv\{\mathbf{z}_1,\ldots,\mathbf{z}_n\},
 \qquad
 \norm{\bar{\mathbf{v}}^{(a)}_t}\le1,
 \qquad
 \ip{\bar{\mathbf{v}}^{(a)}_t}{\wstar}\ge\gamma,\\
 \norm{\bar{\mathbf{v}}^{(a)}_{t+1}-\bar{\mathbf{v}}^{(a)}_t}
 \le\frac{2\mu_t}{S^{(a)}_{t+1}}
 \le2\ln\frac{S^{(a)}_{t+1}}{S^{(a)}_t},
\end{gathered}
 \label{eq:b1-average-step}
\end{equation}
and, for all integers $a<s\le t$,
\begin{equation}
 \norm{\bar{\mathbf{v}}^{(a)}_t-\bar{\mathbf{v}}^{(a)}_s}
 \le\sum_{k=s}^{t-1}\norm{\bar{\mathbf{v}}^{(a)}_{k+1}-\bar{\mathbf{v}}^{(a)}_k}
 \le2\ln\frac{S^{(a)}_t}{S^{(a)}_s}.
 \label{eq:b1-average-telescope}
\end{equation}
\end{lemma}

\begin{proof}
Fix $a$ and, in this proof only, write $S_t$ and $\bar{\mathbf{v}}_t$ for
$S^{(a)}_t$ and $\bar{\mathbf{v}}^{(a)}_t$. Every $\mu_k$ is positive, so
$S_t>0$, and
$\mathbf{w}_t-\mathbf{w}_a=\sum_{k=a}^{t-1}\mu_k\mathbf{v}_k$ by
Lemma~\ref{lem:b1-update}. Thus $\bar{\mathbf{v}}_t$ is a convex combination,
with weights $\mu_k/S_t$, of vectors of the convex hull of the signed
examples; it therefore lies in this convex hull, and \eqref{eq:b1-hull}
gives the norm and margin bounds. Next, $S_{t+1}=S_t+\mu_t$ and
$S_{t+1}\bar{\mathbf{v}}_{t+1}=S_t\bar{\mathbf{v}}_t+\mu_t\mathbf{v}_t$, so
\[
 \bar{\mathbf{v}}_{t+1}-\bar{\mathbf{v}}_t
 =\frac{\mu_t}{S_{t+1}}\big(\mathbf{v}_t-\bar{\mathbf{v}}_t\big),
 \qquad
 \norm{\bar{\mathbf{v}}_{t+1}-\bar{\mathbf{v}}_t}
 \le\frac{\mu_t}{S_{t+1}}\big(\norm{\mathbf{v}_t}+\norm{\bar{\mathbf{v}}_t}\big)
 \le\frac{2\mu_t}{S_{t+1}}.
\]
Moreover, $\mu_t/S_{t+1}=1-S_t/S_{t+1}\le\ln(S_{t+1}/S_t)$, because
$1-y\le-\ln y$ for every $y>0$. This proves \eqref{eq:b1-average-step}. The
first inequality in \eqref{eq:b1-average-telescope} is the triangle
inequality, and the second follows by summing the last bound in
\eqref{eq:b1-average-step} over $s\le k<t$, where the logarithms telescope.
\end{proof}

For $a=0$, since $\mathbf{w}_0=\mathbf{0}$, the sum $S^{(0)}_t$ is the
cumulative coefficient $S_t=\sum_{k<t}\mu_k$ of
Section~\ref{sec:hierarchy} and
$\bar{\mathbf{v}}^{(0)}_t=\mathbf{w}_t/S_t=\boldsymbol{\xi}_t$ is the
normalized iterate; for an interval $I=\{a,\ldots,b\}$, the average
$\bar{\mathbf{v}}^{(a)}_t$ is the local average $\boldsymbol{\zeta}_t$ in
the proof of Lemma~\ref{lem:refine}. At step zero every weight equals
$q_{i,0}=1/2$, so $S_1=\mu_0=\eta/2$, and if $\eta\ge4$ and $U\ge2$, then
$S_U\le2\eta\gamma^{-2}$ by Lemma~\ref{lem:b1-budget} with
$T=U\le\widehat T$. Hence \eqref{eq:b1-average-telescope} with $a=0$,
$s=1$, and $t=U$ gives
\begin{equation}
 \sum_{t=1}^{U-1}\norm{\boldsymbol{\xi}_{t+1}-\boldsymbol{\xi}_t}
 \le2\ln\frac{S_U}{S_1}
 \le2\ln(4\gamma^{-2}).
 \label{eq:b1-path-length}
\end{equation}
When $U\le1$ the sum on the left of \eqref{eq:b1-path-length} is empty, and
its bound by $2\ln(4\gamma^{-2})>0$ holds trivially because $\gamma\le1$.
Together with \eqref{eq:b1-average-step} for $a=0$, this is \eqref{eq:path}.

\subsection{Constructing the hierarchy of intervals}
\label{app:refinement}

This part proves Lemma~\ref{lem:refine} for a general integer $h\ge2$, with
the exact cut points of both constructions, and shows that
Lemmas~\ref{lem:capacity} and~\ref{lem:terminal} hold for every such $h$.
Throughout, $\eta\ge4$ and $\Lambda>10$, and every interval is a set of
consecutive steps in $\mathcal T_+$, so that Lemma~\ref{lem:b1-basic}
provides the budget $\sum_{t<U}\mu_t\le2\eta\gamma^{-2}$ and, at each step
$t$ of the interval, the bounds $\mu_t>1/2$ and
$0\le\ip{\mathbf{w}_t}{\mathbf{v}_t}<\Lambda$. We also use that $\mu_t>0$,
because every weight $q_{i,t}$ in \eqref{eq:weights} is positive, that
$\norm{\mathbf{v}_t}\le1$ at every step (Lemma~\ref{lem:b1-update}), and
that $0<\theta_h<\gamma^2\le1$, because $h\ge2$ and $\gamma\le1$
(Assumption~\ref{ass:data}).
Figure~\ref{fig:hierarchy} gives an overview of the two constructions below
and of the counting argument of Appendix~\ref{app:packing}. In the figure,
$\theta$ and $\nu$ stand for $\theta_h$ and $\nu_h$.


\begin{definition}[$\theta_h$-reference directions]
\label{def:b2-references}
Let $I\subseteq\mathcal T_+$ be a set of consecutive steps. Vectors
$\mathbf{a}_1,\ldots,\mathbf{a}_j$ are \emph{$\theta_h$-reference directions
for $I$} if
\begin{enumerate}[label=\textup{(\alph*)},leftmargin=*,nosep]
 \item $\mathbf{a}_i\in\conv\{\mathbf{z}_1,\ldots,\mathbf{z}_n\}$ for every
 $1\le i\le j$;
 \item $\ip{\mathbf{a}_i}{\mathbf{a}_k}\le\theta_h$ whenever
 $1\le i<k\le j$;
 \item $\ip{\mathbf{a}_i}{\mathbf{v}_t}\le\theta_h$ whenever $1\le i\le j$
 and $t\in I$.
\end{enumerate}
\end{definition}

This is \eqref{eq:references} with $\theta=\theta_h$. Condition (b) is void
when $j=1$. The conditions are local to one interval, and directions
attached to different roots, or to different children of one parent, are
never compared with each other.

\paragraph{Running averages.}
For integers $t\ge a\ge0$ write $S^{(a)}_t:=\sum_{k=a}^{t-1}\mu_k$, an
empty sum when $t=a$, and for $t>a$ write
$\boldsymbol{\zeta}_t:=(\mathbf{w}_t-\mathbf{w}_a)/S^{(a)}_t$; the starting
step $a$ is fixed in each proof. In the notation of
Lemma~\ref{lem:b1-average},
$\boldsymbol{\zeta}_t=\bar{\mathbf{v}}^{(a)}_t$, and
$\boldsymbol{\xi}_t=\mathbf{w}_t/S_t=\bar{\mathbf{v}}^{(0)}_t$ with
$S_t=S^{(0)}_t$. That lemma gives
\begin{equation}
 \boldsymbol{\zeta}_t\in\conv\{\mathbf{z}_1,\ldots,\mathbf{z}_n\}
 \quad(t>a),
 \qquad
 \norm{\boldsymbol{\zeta}_t-\boldsymbol{\zeta}_s}
 \le2\ln\frac{S^{(a)}_t}{S^{(a)}_s}
 \quad(a<s\le t).
 \label{eq:b2-average}
\end{equation}
Both constructions below freeze such an average and start a new interval as
soon as the average has moved by more than $\theta_h/2$.

\paragraph{Root partition.}
At $h=h_*$ the next lemma is Lemma~\ref{lem:refine}(i). Its hypothesis
\eqref{eq:b2-root-test} is the root condition stated at the beginning of
this appendix and enters only in the last display of the proof.

\begin{lemma}[General form of Lemma~\ref{lem:refine}(i)]
\label{lem:b2-roots}
Assume
\begin{equation}
 4\Lambda\le\eta\theta_h.
 \label{eq:b2-root-test}
\end{equation}
There are an integer $M\ge0$ and nonempty sets $J_1,\ldots,J_M$ of
consecutive steps such that
\begin{equation}
\begin{aligned}
 &\mathcal T_+\setminus\{0\}
 =J_1\mathbin{\dot\cup}\cdots\mathbin{\dot\cup}J_M,
 \qquad
 \max J_i<\min J_{i+1}\ \ (1\le i<M),\\
 &M\le|\mathcal T_-|+1+\bigg\lceil\frac{4\ln(4\gamma^{-2})}{\theta_h}\bigg\rceil,
\end{aligned}
 \label{eq:b2-root-partition}
\end{equation}
where $M=0$ when $\mathcal T_+\setminus\{0\}$ is empty. For every
$1\le i\le M$, the vector
$\mathbf{a}_1^{(i)}:=\boldsymbol{\xi}_{\min J_i}$ is a
$\theta_h$-reference direction for $J_i$.
\end{lemma}

\begin{proof}
Put $\mathcal A:=\mathcal T_+\setminus\{0\}$. If $\mathcal A=\varnothing$,
take $M=0$. Otherwise $U\ge2$, because
$\mathcal A\subseteq\{1,\ldots,U-1\}$.

\emph{Runs.} Let $\mathcal K_g=\{s_g,s_g+1,\ldots,e_g\}$, $1\le g\le G$, be
the maximal runs of consecutive integers in $\mathcal A$, in chronological
order. Between two consecutive runs lies a step of
$\{1,\ldots,U-1\}\setminus\mathcal T_+\subseteq\mathcal T_-$, and different
gaps contain different steps, so $G-1\le|\mathcal T_-|$.

\emph{Cut points.} For each $1\le g\le G$ set $c_{g,1}:=s_g$ and define
recursively
\begin{equation}
 c_{g,j+1}:=\min\bigg(\bigg\{t\in\{c_{g,j}+1,\ldots,e_g\}:
 \norm{\boldsymbol{\xi}_t-\boldsymbol{\xi}_{c_{g,j}}}>\frac{\theta_h}{2}\bigg\}
 \cup\{e_g+1\}\bigg).
 \label{eq:b2-root-cuts}
\end{equation}
As long as $c_{g,j}\le e_g$, every element of the set in
\eqref{eq:b2-root-cuts} is at least $c_{g,j}+1$, so
$c_{g,j}<c_{g,j+1}\le e_g+1$. Hence the value $e_g+1$ is reached after
finitely many applications of \eqref{eq:b2-root-cuts},
$m_g:=\min\{j\ge1:\ c_{g,j+1}=e_g+1\}$ is finite, and
$c_{g,j}\le e_g$ for $1\le j\le m_g$ by the minimality of $m_g$. The pieces
\[
 J_{g,j}:=\{c_{g,j},c_{g,j}+1,\ldots,c_{g,j+1}-1\}
 \qquad(1\le j\le m_g)
\]
are therefore nonempty and partition $\mathcal K_g$ in chronological order. A
crossing step $c_{g,j+1}\le e_g$ is the first step of the next piece, so no
step belongs to two pieces. Moreover,
\begin{equation}
 \norm{\boldsymbol{\xi}_t-\boldsymbol{\xi}_{c_{g,j}}}\le\frac{\theta_h}{2}
 \qquad(t\in J_{g,j}),
 \label{eq:b2-root-stay}
\end{equation}
because the left side vanishes at $t=c_{g,j}$, while a step $t$ with
$c_{g,j}<t<c_{g,j+1}$ lies in the index range of \eqref{eq:b2-root-cuts}
and precedes the minimum, so it fails the strict inequality there.

\emph{Counting.} For $1\le j<m_g$ we have $c_{g,j+1}\le e_g$, so
\eqref{eq:b2-root-cuts} and the triangle inequality give
\[
 \frac{\theta_h}{2}
 <\norm{\boldsymbol{\xi}_{c_{g,j+1}}-\boldsymbol{\xi}_{c_{g,j}}}
 \le\sum_{t=c_{g,j}}^{c_{g,j+1}-1}
 \norm{\boldsymbol{\xi}_{t+1}-\boldsymbol{\xi}_t}.
\]
The summation range is the piece $J_{g,j}$ itself, and the pieces are
pairwise disjoint subsets of $\{1,\ldots,U-1\}$. There are $M-G$ pairs
$(g,j)$ with $1\le j<m_g$, where $M:=\sum_{g=1}^{G}m_g$. Summing over these
pairs and applying \eqref{eq:b1-path-length},
\[
 \frac{(M-G)\theta_h}{2}
 \le\sum_{t=1}^{U-1}\norm{\boldsymbol{\xi}_{t+1}-\boldsymbol{\xi}_t}
 \le2\ln(4\gamma^{-2}).
\]
Hence
$M-G\le4\ln(4\gamma^{-2})/\theta_h\le\lceil4\ln(4\gamma^{-2})/\theta_h\rceil$,
and together with $G\le|\mathcal T_-|+1$ this is the bound on $M$.
Relabeling the $M$ pieces chronologically as $J_1,\ldots,J_M$ gives
\eqref{eq:b2-root-partition}.

\emph{The reference direction.} Fix $1\le i\le M$, put $c:=\min J_i$, and
let $t\in J_i$. By \eqref{eq:b2-average} with $a=0$, the vector
$\mathbf{a}_1^{(i)}=\boldsymbol{\xi}_c$ lies in the convex hull of the
signed examples, which is condition (a) of
Definition~\ref{def:b2-references}; condition (b) is void. Since $t\ge1$, we
have $S_t\ge S_1=\mu_0=\eta/2$, as computed before
\eqref{eq:b1-path-length}. Writing
$\boldsymbol{\xi}_c=\boldsymbol{\xi}_t+(\boldsymbol{\xi}_c-\boldsymbol{\xi}_t)$
and using the Cauchy--Schwarz inequality with $\norm{\mathbf{v}_t}\le1$,
then $\ip{\mathbf{w}_t}{\mathbf{v}_t}<\Lambda$ and \eqref{eq:b2-root-stay},
and finally \eqref{eq:b2-root-test},
\[
 \ip{\mathbf{a}_1^{(i)}}{\mathbf{v}_t}
 \le\frac{\ip{\mathbf{w}_t}{\mathbf{v}_t}}{S_t}
 +\norm{\boldsymbol{\xi}_c-\boldsymbol{\xi}_t}
 <\frac{2\Lambda}{\eta}+\frac{\theta_h}{2}
 \le\theta_h,
\]
which is condition (c).
\end{proof}

\paragraph{Refinement.}
At $h=h_*$ the next lemma is Lemma~\ref{lem:refine}(ii): the quantity $R_h$
is the $R$ of the main text, \eqref{eq:b2-partition} is \eqref{eq:R}, and
\eqref{eq:b2-separation} is \eqref{eq:separation}.

\begin{lemma}[General form of Lemma~\ref{lem:refine}(ii)]
\label{lem:b2-refine}
Let $I=\{a,a+1,\ldots,b\}\subseteq\mathcal T_+$ carry the
$\theta_h$-reference directions $\mathbf{a}_1,\ldots,\mathbf{a}_j$, let
$R_h:=1+\lceil4\Lambda/\theta_h\rceil$, and let $I_0$ be the set of steps
$t\in I$ with $S^{(a)}_t<2\Lambda/\theta_h$. There are an integer $m\ge0$
and nonempty sets $I_1,\ldots,I_m$ of consecutive steps such that
\begin{equation}
\begin{aligned}
 &I=I_0\mathbin{\dot\cup}I_1\mathbin{\dot\cup}\cdots\mathbin{\dot\cup}I_m,
 \qquad
 \max I_i<\min I_{i+1}\ \ (0\le i<m),\\
 &|I_0|\le\bigg\lceil\frac{4\Lambda}{\theta_h}\bigg\rceil\le R_h,
 \qquad
 m\le R_h;
\end{aligned}
 \label{eq:b2-partition}
\end{equation}
the case $m=0$ is allowed, with $I_0=I$. For $1\le i\le m$ put
$b_i:=\min I_i$ and
\begin{equation}
 \sigma_i:=\sum_{t=a}^{b_i-1}\mu_t=S^{(a)}_{b_i},
 \qquad
 \mathbf{a}_{j+1}^{(i)}
 :=\frac1{\sigma_i}\sum_{k=a}^{b_i-1}\mu_k\mathbf{v}_k
 =\boldsymbol{\zeta}_{b_i}.
 \label{eq:b2-new-direction}
\end{equation}
Then $\mathbf{a}_1,\ldots,\mathbf{a}_j,\mathbf{a}_{j+1}^{(i)}$ are
$\theta_h$-reference directions for $I_i$, and
\begin{equation}
 \sigma_i\ge\frac{2\Lambda}{\theta_h}\ \ (1\le i\le m),
 \qquad
 \ln\frac{\sigma_{i+1}}{\sigma_i}>\frac{\theta_h}{4}\ \ (1\le i<m).
 \label{eq:b2-separation}
\end{equation}
Moreover,
\begin{equation}
 R_h\le\frac{6\Lambda}{\theta_h}.
 \label{eq:b2-R-simple}
\end{equation}
\end{lemma}

\begin{proof}
Since $\lceil x\rceil<x+1$ for every real $x$ and $\theta_h<1<\Lambda$,
we have $R_h<2+4\Lambda/\theta_h<6\Lambda/\theta_h$, which is
\eqref{eq:b2-R-simple}. We also use twice that an integer smaller than
$1+x$ is at most $\lceil x\rceil$.

\emph{The prefix.} Since every $\mu_k$ is positive, $t\mapsto S^{(a)}_t$ is
increasing on $\{a,\ldots,b+1\}$, and $S^{(a)}_a=0$. Hence
$I_0=\{a,\ldots,b_1-1\}$, where
\begin{equation}
 b_1:=\min\bigg(\bigg\{t\in\{a+1,\ldots,b\}:
 S^{(a)}_t\ge\frac{2\Lambda}{\theta_h}\bigg\}\cup\{b+1\}\bigg).
 \label{eq:b2-threshold}
\end{equation}
The step $b_1-1$ belongs to $I_0$, and $\mu_k>1/2$ for $a\le k\le b_1-2$,
so
\[
 \frac{|I_0|-1}{2}=\frac{b_1-1-a}{2}
 \le S^{(a)}_{b_1-1}<\frac{2\Lambda}{\theta_h}.
\]
The coefficient in this display is measured before the last step of the
prefix. That step may carry the sum past the threshold, so
$\sum_{t\in I_0}\mu_t$ itself need not be smaller than $2\Lambda/\theta_h$,
and the argument does not use such a bound. The display gives
$|I_0|<1+4\Lambda/\theta_h$, hence $|I_0|\le\lceil4\Lambda/\theta_h\rceil\le R_h$.
If $b_1=b+1$, take $m=0$; then $I_0=I$, \eqref{eq:b2-partition} holds, and
the remaining assertions are void. Assume from now on that $b_1\le b$.

\emph{The local radial bound.} Let $b_1\le t\le b$. Every signed margin at
the step $a\in\mathcal T_+$ is nonnegative and $\mathbf{v}_t$ is a convex
combination of the signed examples, so
$\ip{\mathbf{w}_a}{\mathbf{v}_t}\ge0$. Together with
$\ip{\mathbf{w}_t}{\mathbf{v}_t}<\Lambda$ and
$S^{(a)}_t\ge S^{(a)}_{b_1}\ge2\Lambda/\theta_h$, this gives
\begin{equation}
 \ip{\boldsymbol{\zeta}_t}{\mathbf{v}_t}
 =\frac{\ip{\mathbf{w}_t}{\mathbf{v}_t}-\ip{\mathbf{w}_a}{\mathbf{v}_t}}{S^{(a)}_t}
 \le\frac{\ip{\mathbf{w}_t}{\mathbf{v}_t}}{S^{(a)}_t}
 <\frac{\Lambda}{S^{(a)}_t}
 \le\frac{\theta_h}{2}
 \qquad(b_1\le t\le b),
 \label{eq:b2-half}
\end{equation}
which is \eqref{eq:local-radial}.

\emph{The children.} Starting from $b_1$, define recursively
\begin{equation}
 b_{i+1}:=\min\bigg(\bigg\{t\in\{b_i+1,\ldots,b\}:
 \norm{\boldsymbol{\zeta}_t-\boldsymbol{\zeta}_{b_i}}>\frac{\theta_h}{2}\bigg\}
 \cup\{b+1\}\bigg),
 \label{eq:b2-cuts}
\end{equation}
let $m:=\min\{i\ge1:\ b_{i+1}=b+1\}$, which is finite because
$b_i<b_{i+1}\le b+1$ as long as $b_i\le b$, and put
\[
 I_i:=\{b_i,b_i+1,\ldots,b_{i+1}-1\}
 \qquad(1\le i\le m).
\]
Since $b_i\le b$ for $1\le i\le m$ by the minimality of $m$, these sets are
nonempty, satisfy $\min I_i=b_i$, and partition
$\{b_1,\ldots,b\}=I\setminus I_0$ in chronological order, which gives the
first line of \eqref{eq:b2-partition}. A crossing step $b_{i+1}\le b$ is the
first step of $I_{i+1}$ and does not belong to $I_i$, so no step is counted
twice. As for \eqref{eq:b2-root-stay}, the minimality in \eqref{eq:b2-cuts}
gives
\begin{equation}
 \norm{\boldsymbol{\zeta}_t-\boldsymbol{\zeta}_{b_i}}\le\frac{\theta_h}{2}
 \qquad(t\in I_i,\ 1\le i\le m).
 \label{eq:b2-stay}
\end{equation}

\emph{Separation and the number of children.} By monotonicity and
\eqref{eq:b2-threshold},
$\sigma_i=S^{(a)}_{b_i}\ge S^{(a)}_{b_1}\ge2\Lambda/\theta_h$ for every
$1\le i\le m$. For $1\le i<m$ we have $b_{i+1}\le b$, so
\eqref{eq:b2-cuts} and \eqref{eq:b2-average} give
\[
 \frac{\theta_h}{2}
 <\norm{\boldsymbol{\zeta}_{b_{i+1}}-\boldsymbol{\zeta}_{b_i}}
 \le2\ln\frac{S^{(a)}_{b_{i+1}}}{S^{(a)}_{b_i}}
 =2\ln\frac{\sigma_{i+1}}{\sigma_i},
\]
which completes \eqref{eq:b2-separation}. Summing over $1\le i<m$, the
logarithms telescope. Since $b_m\le b<U$, the budget gives
$\sigma_m\le\sum_{k<U}\mu_k\le2\eta\gamma^{-2}$, and
$\sigma_1\ge2\Lambda/\theta_h$, so
\[
 \frac{(m-1)\theta_h}{4}
 \le\ln\frac{\sigma_m}{\sigma_1}
 \le\ln\bigg(2\eta\gamma^{-2}\cdot\frac{\theta_h}{2\Lambda}\bigg)
 =\ln\frac{\eta}{h\Lambda}
 <\Lambda,
\]
where the equality is $\gamma^{-2}\theta_h=1/h$ and the last inequality uses
$h\Lambda>1$. Hence $m<1+4\Lambda/\theta_h$, so
$m\le\lceil4\Lambda/\theta_h\rceil\le R_h$, which completes
\eqref{eq:b2-partition}.

\emph{The reference directions of a child.} Fix $1\le i\le m$. The new
direction $\mathbf{a}_{j+1}^{(i)}=\boldsymbol{\zeta}_{b_i}$ of
\eqref{eq:b2-new-direction} lies in the convex hull of the signed examples
by \eqref{eq:b2-average}, which is condition (a) of
Definition~\ref{def:b2-references}. The inherited
directions satisfy conditions (a) and (b) among themselves by hypothesis,
and condition (c) on $I_i$ because $I_i\subseteq I$. For $1\le s\le j$,
condition (c) on $I$ gives $\ip{\mathbf{a}_s}{\mathbf{v}_k}\le\theta_h$
for $a\le k<b_i$, and the weights $\mu_k/\sigma_i$ are positive and sum to
one, so
\[
 \ip{\mathbf{a}_s}{\mathbf{a}_{j+1}^{(i)}}
 =\frac1{\sigma_i}\sum_{k=a}^{b_i-1}\mu_k\ip{\mathbf{a}_s}{\mathbf{v}_k}
 \le\theta_h.
\]
Thus the inherited inner-product bounds pass through the convex average
exactly, and no error accumulates with the depth. For $t\in I_i$ we have
$b_1\le b_i\le t\le b$, and the Cauchy--Schwarz inequality with
$\norm{\mathbf{v}_t}\le1$, \eqref{eq:b2-half}, and \eqref{eq:b2-stay} give
\[
 \ip{\mathbf{a}_{j+1}^{(i)}}{\mathbf{v}_t}
 \le\ip{\boldsymbol{\zeta}_t}{\mathbf{v}_t}
 +\norm{\boldsymbol{\zeta}_{b_i}-\boldsymbol{\zeta}_t}
 <\frac{\theta_h}{2}+\frac{\theta_h}{2}=\theta_h.
\]
The last two displays are \eqref{eq:caps}, and they verify conditions (b)
and (c) for the new direction. Hence
$\mathbf{a}_1,\ldots,\mathbf{a}_j,\mathbf{a}_{j+1}^{(i)}$ are
$\theta_h$-reference directions for $I_i$.
\end{proof}

All the $\sigma_i$ in \eqref{eq:b2-separation} are measured from the first
step $a$ of the parent, and Appendix~\ref{app:packing} uses the separation
in this form.

\paragraph{Capacity and terminal intervals.}
At $h=h_*$, parts (i) and (ii) of the next lemma are
Lemmas~\ref{lem:capacity} and~\ref{lem:terminal}, which are proved in full
in Section~\ref{sec:hierarchy}.

\begin{lemma}[General form of Lemmas~\ref{lem:capacity} and~\ref{lem:terminal}]
\label{lem:b2-capacity}\label{lem:b2-terminal}
\textup{(i)} If
$\mathbf{b}_1,\ldots,\mathbf{b}_m\in\conv\{\mathbf{z}_1,\ldots,\mathbf{z}_n\}$
satisfy $\ip{\mathbf{b}_i}{\mathbf{b}_k}\le\theta_h$ for all $i\ne k$, then
\begin{equation}
 m\le\min\bigg\{r,\ \bigg\lfloor\frac{1-\theta_h}{\gamma^{2}-\theta_h}\bigg\rfloor\bigg\}
 \le\nu_h+2.
 \label{eq:b2-capacity}
\end{equation}
\textup{(ii)} If an interval $I\subseteq\mathcal T_+$ carries $\nu_h+1$
reference directions in the sense of Definition~\ref{def:b2-references},
then
\[
 |I|\le\bigg\lceil\frac{4}{\theta_h}\bigg\rceil
 \le\frac{6}{\theta_h}=6h\gamma^{-2}.
\]
\end{lemma}

\begin{proof}
The proofs of Lemmas~\ref{lem:capacity} and~\ref{lem:terminal} use $\theta$
only through $0<\theta<\gamma^2$ and, in the last step of
Lemma~\ref{lem:capacity}, through the identity \eqref{eq:floor-identity},
which identifies the floor in \eqref{eq:b2-capacity} with the floor in the
definition \eqref{eq:depth} of $\nu_h$, so that, by \eqref{eq:depth}, the
minimum in \eqref{eq:b2-capacity} is at most $\nu_h+2$. The facts about the
trajectory used in the proof of Lemma~\ref{lem:terminal} are the gradient
identity of Lemma~\ref{lem:b1-update} and the bounds
$\mu_t\ge\eta L(\mathbf{w}_t)/2$ and $L(\mathbf{w}_t)>1/\eta$ on
$\mathcal T_+$ of Lemma~\ref{lem:b1-basic}, none of which involves $h$. Both
proofs therefore apply with $\theta_h$ and $\nu_h$ in place of $\theta$ and
$\nu$, the second one using part~(i), and give \eqref{eq:b2-capacity} and
$|I|\le\lceil4/\theta_h\rceil$; finally, $\theta_h<1$ gives
$\lceil4/\theta_h\rceil<4/\theta_h+1<6/\theta_h=6h\gamma^{-2}$.
\end{proof}

A root interval carries one reference direction, and each application of
Lemma~\ref{lem:b2-refine} adds one. An interval carrying $\nu_h+1$
reference directions is called \emph{terminal}; it is not refined further,
and Lemma~\ref{lem:b2-terminal}(ii), whose proof rests on the capacity
bound of Lemma~\ref{lem:b2-capacity}(i), bounds its length.
Appendix~\ref{app:packing} organizes the resulting intervals into trees and
counts them.

\subsection{Counting intervals}
\label{app:packing}

This part proves, for a general integer $h\ge2$, the inequality
\eqref{eq:packing} for the children of an interval, the descendant count of Lemma~\ref{lem:factorial}, and
the per-root count \eqref{eq:per-root}, which
Appendix~\ref{app:completion} adds over the root intervals of
Lemma~\ref{lem:b2-roots}. Throughout this part, $\eta\ge4$ and
$\Lambda=\ln\eta>10$, and we use the bound
$0\le\nu_h<2\gamma^{-2}$ of \eqref{eq:depth-range}.

\paragraph{The tree of a root interval.}
Let $J$ be a nonempty set of consecutive steps in $\mathcal T_+$ carrying
one $\theta_h$-reference direction (Definition~\ref{def:b2-references}),
for instance a root interval of Lemma~\ref{lem:b2-roots}. Recursive
refinement of $J$ produces a rooted tree. A node is an interval together
with its reference directions, and the root $J$ has depth zero. A terminal
node, that is, a node carrying $\nu_h+1$ reference directions
(Appendix~\ref{app:refinement}), is not refined. A node $I$ carrying at
most $\nu_h$ reference directions is
partitioned by Lemma~\ref{lem:b2-refine} into a prefix $I_0$ and children
$I_1,\ldots,I_m$ as in \eqref{eq:b2-partition}, and each child, together
with the reference directions of $I$ and the new one supplied by the lemma,
is a node at the next depth. We write
$\operatorname{succ}(I):=\{I_1,\ldots,I_m\}$ for such a node and
$\operatorname{succ}(I):=\varnothing$ for a terminal one. The case $m=0$ is
allowed; then $I_0=I$, $\operatorname{succ}(I)=\varnothing$, and that
branch ends. The root carries one reference direction and each child gains
exactly one, so a node at depth $j$ carries exactly $j+1$ of them. Hence
the nodes at depth $\nu_h$ are exactly the terminal ones, every node at
smaller depth is refined, and when $\nu_h=0$ the tree is the root alone.

For a node $I$ and an integer $j\ge0$ let $D_j(I)$ be the number of
descendants of $I$ at depth $j$ below $I$. Thus $D_0(I)=1$ and
\begin{equation}
 D_{j+1}(I)=\sum_{I'\in\operatorname{succ}(I)}D_j(I')
 \qquad(j\ge0),
 \label{eq:b3-recursion}
\end{equation}
where an empty sum is zero.

For a nonempty interval $I$ of the construction define its coefficient and
its logarithmic size by
\begin{equation}
 \mu(I):=\sum_{t\in I}\mu_t,
 \qquad
 \chi_h(I):=\max\bigg\{0,\ \ln\frac{\theta_h\mu(I)}{2\Lambda}\bigg\},
 \label{eq:b3-chi}
\end{equation}
which is \eqref{eq:chi} with $\theta=\theta_h$; the normalization is the
coefficient threshold $2\Lambda/\theta_h$ of Lemma~\ref{lem:b2-refine}.
Since $I\subseteq\{0,\ldots,U-1\}$ and every $\mu_t$ is positive, the first
bound in \eqref{eq:b1-budget} (Lemma~\ref{lem:b1-basic}, which requires
$\eta\ge4$) gives
$\mu(I)\le2\eta\gamma^{-2}$, so
$\theta_h\mu(I)/(2\Lambda)
\le(\gamma^{2}/h)(2\eta\gamma^{-2})/(2\Lambda)=\eta/(h\Lambda)$, and
$h\Lambda\ge1$ gives
\begin{equation}
 0\le\chi_h(I)\le\max\bigg\{0,\ \ln\frac{\eta}{h\Lambda}\bigg\}
 \le\ln\eta=\Lambda.
 \label{eq:b3-chi-range}
\end{equation}
For integers $k\ge0$ let $E_k(s):=\sum_{i=0}^{k}s^{i}/i!$ be the
truncated exponential. On $[0,\infty)$ every $E_k$ is positive and
nondecreasing, and $E_k\le E_{k+1}$. When $\nu_h\ge1$ put
$\alpha:=2\nu_h/\Lambda$.

\paragraph{Counting siblings.}
By \eqref{eq:b2-separation}, the children of a node start at coefficients
whose logarithms, measured on the coefficient axis of the parent, are more
than $\theta_h/4$ apart (Figure~\ref{fig:hierarchy}(b)). Counting the siblings jointly along this axis gives
the next lemma, which is \eqref{eq:packing} with $\theta=\theta_h$ and
$\nu=\nu_h$.

\begin{lemma}[Bound on the children of an interval]
\label{lem:b3-packing}
Let $\nu_h\ge1$, and let $I=\{a,\ldots,b\}$ be a nonterminal node with
children $I_1,\ldots,I_m$ given by Lemma~\ref{lem:b2-refine}. Then, for
every integer $k\ge0$,
\begin{equation}
 \sum_{i=1}^{m}E_k\big(\alpha\chi_h(I_i)\big)
 \le\frac{6}{\alpha\theta_h}\,E_{k+1}\big(\alpha\chi_h(I)\big).
 \label{eq:b3-packing}
\end{equation}
\end{lemma}

\begin{proof}
If $m=0$, the left side is zero and the right side is positive. Assume
$m\ge1$, and let $b_i:=\min I_i$ and $\sigma_i:=\sum_{t=a}^{b_i-1}\mu_t$
for $1\le i\le m$, as in Lemma~\ref{lem:b2-refine}.

By
\eqref{eq:b2-partition}, the sets $I_0,I_1,\ldots,I_m$ are consecutive
pieces of the consecutive set $I$ with $\max I_i<\min I_{i+1}$, so
$I_i=\{b_i,\ldots,b_{i+1}-1\}$ for $1\le i<m$. Hence a completed child
satisfies
\[
 \mu(I_i)=\sigma_{i+1}-\sigma_i\le\sigma_{i+1}
 \qquad(1\le i<m),
\]
while the last child satisfies only $\mu(I_m)\le\mu(I)$, because
$I_m\subseteq I$ and every $\mu_t$ is positive. Define
\[
 u_j:=\ln\frac{\theta_h\sigma_j}{2\Lambda}
 \qquad(2\le j\le m).
\]
The first part of \eqref{eq:b2-separation} gives $u_j\ge0$, and the second
part gives $u_{j+1}-u_j=\ln(\sigma_{j+1}/\sigma_j)>\theta_h/4$ for
$2\le j<m$. Since $\{a,\ldots,b_j-1\}\subseteq I$, we have
$\sigma_j\le\mu(I)$, and therefore $u_j\le\chi_h(I)$. By
\eqref{eq:b3-chi}, $\chi_h$ is a nondecreasing function of the coefficient,
so the two bounds on $\mu(I_i)$ give
\[
 \chi_h(I_i)\le\max\{0,u_{i+1}\}=u_{i+1}\quad(1\le i<m),
 \qquad
 \chi_h(I_m)\le\chi_h(I).
\]

Let
$f\colon[0,\infty)\to[0,\infty)$ be nondecreasing, hence integrable on
bounded intervals. The intervals $[u_j,u_j+\theta_h/4]$, $2\le j\le m$,
are pairwise disjoint, because consecutive left endpoints differ by more
than $\theta_h/4$, and they are contained in
$[0,\chi_h(I)+\theta_h/4]$. Hence the bounds on $\chi_h(I_i)$, the
monotonicity of $f$, and
$(\theta_h/4)f(u_j)\le\int_{u_j}^{u_j+\theta_h/4}f(u)\,du$ give
\begin{equation}
 \sum_{i=1}^{m}f\big(\chi_h(I_i)\big)
 \le f\big(\chi_h(I)\big)+\sum_{j=2}^{m}f(u_j)
 \le f\big(\chi_h(I)\big)
 +\frac{4}{\theta_h}\int_{0}^{\chi_h(I)+\theta_h/4}f(u)\,du,
 \label{eq:b3-riemann}
\end{equation}
which is \eqref{eq:riemann} with $\theta=\theta_h$.

Take $f(u):=E_k(\alpha u)$, which is
nonnegative and nondecreasing on $[0,\infty)$. Since $E_{k+1}'=E_k$ and
$E_{k+1}(0)=1$,
\[
 \int_{0}^{v}E_k(\alpha u)\,du=\frac{E_{k+1}(\alpha v)-1}{\alpha}
 \qquad(v\ge0).
\]
Moreover $E_{k+1}'=E_k\le E_{k+1}$ on $[0,\infty)$, so
$\frac{d}{ds}\big(e^{-s}E_{k+1}(s)\big)=e^{-s}\big(E_k(s)-E_{k+1}(s)\big)\le0$
and $s\mapsto e^{-s}E_{k+1}(s)$ is nonincreasing there, that is,
\begin{equation}
 E_{k+1}(u+v)\le e^{v}E_{k+1}(u)
 \qquad(u,v\ge0).
 \label{eq:b3-shift}
\end{equation}
We evaluate the integral at the upper limit $\chi_h(I)+\theta_h/4$,
substitute it into \eqref{eq:b3-riemann}, drop the term
$-4/(\alpha\theta_h)<0$, and apply $E_k\le E_{k+1}$ and
\eqref{eq:b3-shift} with $u=\alpha\chi_h(I)$ and $v=\alpha\theta_h/4$. This
gives
\begin{align*}
 \sum_{i=1}^{m}E_k\big(\alpha\chi_h(I_i)\big)
 &\le E_k\big(\alpha\chi_h(I)\big)
 +\frac{4}{\alpha\theta_h}
 E_{k+1}\bigg(\alpha\chi_h(I)+\frac{\alpha\theta_h}{4}\bigg)\\
 &\le\frac{\alpha\theta_h+4e^{\alpha\theta_h/4}}{\alpha\theta_h}\,
 E_{k+1}\big(\alpha\chi_h(I)\big).
\end{align*}

By $\nu_h<2\gamma^{-2}$, $h\ge2$, and $\Lambda>10$, we
have $0<\alpha\theta_h=2\nu_h\gamma^{2}/(h\Lambda)<4/(h\Lambda)<1/5$. Since
$e^{s}\le(1-s)^{-1}$ for $0\le s<1$, we get
$e^{\alpha\theta_h/4}<(1-1/20)^{-1}<5/4$, hence
$\alpha\theta_h+4e^{\alpha\theta_h/4}<1/5+5<6$, and \eqref{eq:b3-packing}
follows.
\end{proof}

\paragraph{Counting descendants.}
Iterating Lemma~\ref{lem:b3-packing} down the tree gives the general form
of Lemma~\ref{lem:factorial}, which is the case $h=h_*$.

\begin{lemma}[General form of Lemma~\ref{lem:factorial}]
\label{lem:b3-factorial}
Let $\nu_h\ge1$. For every node $I$ of the construction and every
$0\le j\le\nu_h$,
\begin{equation}
 D_j(I)\le2\,\frac{(6\Lambda/\theta_h)^{j}}{j!}
 =2\,\frac{(6h\gamma^{-2}\Lambda)^{j}}{j!}.
 \label{eq:b3-factorial}
\end{equation}
\end{lemma}

\begin{proof}
We first prove the weighted estimate
\begin{equation}
 D_j(I)\le\bigg(\frac{6}{\alpha\theta_h}\bigg)^{j}E_j\big(\alpha\chi_h(I)\big)
 \label{eq:b3-weighted}
\end{equation}
by induction on $j\ge0$, simultaneously for all nodes $I$ of the
construction. For $j=0$ both sides equal one. Assume that
\eqref{eq:b3-weighted} holds at $j$ for every node, and fix a node $I$. If
$\operatorname{succ}(I)=\varnothing$, then $D_{j+1}(I)=0$ by
\eqref{eq:b3-recursion}, and the right side of \eqref{eq:b3-weighted} at
$j+1$ is positive. Otherwise $I$ is nonterminal with children
$I_1,\ldots,I_m$, and \eqref{eq:b3-recursion}, the induction hypothesis
applied to every child, and Lemma~\ref{lem:b3-packing} with $k=j$ give
\[
 D_{j+1}(I)
 =\sum_{i=1}^{m}D_j(I_i)
 \le\bigg(\frac{6}{\alpha\theta_h}\bigg)^{j}
 \sum_{i=1}^{m}E_j\big(\alpha\chi_h(I_i)\big)
 \le\bigg(\frac{6}{\alpha\theta_h}\bigg)^{j+1}E_{j+1}\big(\alpha\chi_h(I)\big).
\]
This proves \eqref{eq:b3-weighted} for every $j\ge0$.

Now let $0\le j\le\nu_h$. By \eqref{eq:b3-chi-range} and the monotonicity of
$E_j$, we have $E_j(\alpha\chi_h(I))\le E_j(\alpha\Lambda)$, and
$\alpha\Lambda=2\nu_h\ge2j$. The terms $a_k:=(\alpha\Lambda)^{k}/k!$ of
$E_j(\alpha\Lambda)$ satisfy
$a_{k-1}/a_k=k/(\alpha\Lambda)\le j/(2\nu_h)\le1/2$ for $1\le k\le j$, so $a_{j-i}\le2^{-i}a_j$ for $0\le i\le j$ and
\[
 E_j(\alpha\Lambda)=\sum_{i=0}^{j}a_{j-i}
 \le a_j\sum_{i=0}^{j}2^{-i}<2\,\frac{(\alpha\Lambda)^{j}}{j!}.
\]
Substituting into \eqref{eq:b3-weighted}, the powers of $\alpha$ cancel,
which gives $D_j(I)\le2(6\Lambda/\theta_h)^{j}/j!$, and
$1/\theta_h=h\gamma^{-2}$ gives the second form in \eqref{eq:b3-factorial}.
\end{proof}

\paragraph{Counting the steps of a root interval.}
Combining Lemma~\ref{lem:b3-factorial} with the lengths of prefixes and
terminal nodes gives the general form of \eqref{eq:per-root}, which is the
case $h=h_*$ of the following bound.

\begin{proposition}[Number of steps of a root interval]
\label{prop:b3-root-count}
Let $J$ be a nonempty set of consecutive steps in $\mathcal T_+$ carrying
one $\theta_h$-reference direction. Then
\begin{equation}
 |J|\le\big(3\nu_h+12h\gamma^{-2}\big)\,
 \frac{(6h\gamma^{-2}\Lambda)^{\nu_h}}{\nu_h!}.
 \label{eq:b3-per-root}
\end{equation}
\end{proposition}

\begin{proof}
If $\nu_h=0$, then $J$ is terminal, and Lemma~\ref{lem:b2-terminal}(ii)
gives
$|J|\le\lceil4/\theta_h\rceil\le6/\theta_h=6h\gamma^{-2}$, which is smaller
than the right side $12h\gamma^{-2}$ of \eqref{eq:b3-per-root}.

Let $\nu_h\ge1$, consider the tree of $J$ constructed above, and put
$z:=6\Lambda/\theta_h=6h\gamma^{-2}\Lambda$. For
every $0\le j\le\nu_h$, the set $J$ is the disjoint union of the prefixes of
the nodes at depths below $j$ and of the nodes at depth $j$. This is
trivial for $j=0$, and the step from $j<\nu_h$ to $j+1$ replaces every node
$I$ at depth $j$, which is nonterminal, by its partition
\eqref{eq:b2-partition} into $I_0$ and the nodes of
$\operatorname{succ}(I)$; a node without children contributes only its
prefix $I_0=I$. At $j=\nu_h$ the nodes at depth $j$ are the terminal ones.
There are $D_j(J)$ nodes at depth $j$, every prefix has at most $R_h$ steps
by \eqref{eq:b2-partition}, and every terminal node has at most
$\lceil4/\theta_h\rceil$ steps by Lemma~\ref{lem:b2-terminal}(ii).
Therefore
\begin{equation}
 |J|\le R_h\sum_{j=0}^{\nu_h-1}D_j(J)
 +\bigg\lceil\frac{4}{\theta_h}\bigg\rceil D_{\nu_h}(J).
 \label{eq:b3-partition-count}
\end{equation}
By $\nu_h<2\gamma^{-2}$, $h\ge2$, and $\Lambda>10$, we have
$z/\nu_h>3h\Lambda>18$, so each term of $\sum_{j<\nu_h}z^{j}/j!$ is less
than $1/18$ of the next, and the sum is at most
$(18/17)\,z^{\nu_h-1}/(\nu_h-1)!$. Lemma~\ref{lem:b3-factorial} therefore
gives
\[
 \sum_{j=0}^{\nu_h-1}D_j(J)
 \le2\sum_{j=0}^{\nu_h-1}\frac{z^{j}}{j!}
 <3\,\frac{z^{\nu_h-1}}{(\nu_h-1)!},
 \qquad
 D_{\nu_h}(J)\le2\,\frac{z^{\nu_h}}{\nu_h!}.
\]
Finally, $R_h\le6\Lambda/\theta_h=z$ by
\eqref{eq:b2-R-simple} and $\lceil4/\theta_h\rceil\le6/\theta_h$ by
Lemma~\ref{lem:b2-terminal}(ii).
Substituting these four bounds into \eqref{eq:b3-partition-count} and using
$z\cdot z^{\nu_h-1}/(\nu_h-1)!=\nu_hz^{\nu_h}/\nu_h!$ gives
\[
 |J|\le3\nu_h\,\frac{z^{\nu_h}}{\nu_h!}
 +\frac{12}{\theta_h}\,\frac{z^{\nu_h}}{\nu_h!}
 =\big(3\nu_h+12h\gamma^{-2}\big)\frac{(6h\gamma^{-2}\Lambda)^{\nu_h}}{\nu_h!},
\]
which is \eqref{eq:b3-per-root}.
\end{proof}

\subsection{Stepsize conditions and completion of the proof}
\label{app:completion}

This part proves that every $\eta\ge\eta_h^0$ satisfies the stepsize conditions
\eqref{eq:onset-conditions} for general $h$ (Lemma~\ref{lem:b4-onset}),
completes the proof of Theorem~\ref{thm:general} and with it that of the
low-margin case of Theorem~\ref{thm:main}(i), and derives the factorial
bound that is available before the absorption
(Proposition~\ref{prop:b4-factorial}).
Throughout, Assumption~\ref{ass:data} holds, the integer $h\ge2$ is fixed,
$\Lambda=\ln\eta$, and the only condition on the stepsize is
$\eta\ge\eta_h^0$. The restriction $0<\gamma<1/\sqrt2$ of
Theorem~\ref{thm:general} is used only in the proofs of the theorem and of
Proposition~\ref{prop:b4-factorial}, where it gives $\gamma^{-2}>2$;
Lemma~\ref{lem:b4-onset}, on which
Appendices~\ref{app:trajectory}--\ref{app:packing} rely for their standing
assumptions, uses only $\gamma\le1$ from Assumption~\ref{ass:data}.

Two facts about the constants \eqref{eq:depth}--\eqref{eq:onset} are used
several times. First, $\gamma\le1$, $h\ge2$, and $n\ge1$ give
\begin{equation}
 C_h\ge80\gamma^{-4}hn\ge160n\gamma^{-4}>4n\gamma^{-2}\ge4\gamma^{-2}>1.
 \label{eq:b4-constants}
\end{equation}
Second, let $\nu_h\ge1$ and put $b_h:=\max\{2,\nu_h\}$. The quotient
inside the ceiling of \eqref{eq:exponent} is positive and $p_h\ge\nu_h+1$,
as observed after that display, and a ceiling is at least its argument, so
\begin{equation}
 p_h-\nu_h\ge\frac{\nu_h\ln(18h\gamma^{-2}/\nu_h)}{\ln b_h}>0
 \qquad(\nu_h\ge1).
 \label{eq:b4-ceiling}
\end{equation}
Since $p_h=0$ when $\nu_h=0$, we have $p_h\ge0$ in every case.

\begin{lemma}[Stepsize conditions]
\label{lem:b4-onset}
Every $\eta\ge\eta_h^0$ satisfies
\begin{equation}
 \eta\ge4,
 \qquad
 \Lambda>10,
 \qquad
 4h\gamma^{-2}\Lambda\le\eta,
 \qquad
 C_h\Lambda^{p_h}<\widehat T.
 \label{eq:b4-onset}
\end{equation}
The third condition is the root condition $4\Lambda\le\eta\theta_h$, and at
$h=h_*$ the four conditions are \eqref{eq:onset-conditions}.
\end{lemma}

\begin{proof}
Write $K:=K_h$, $\eta_*:=\eta_h^0$, and $s:=1+\ln K$. The terms
$h\gamma^{-2}$, $p_h$, and $\ln(2\gamma^{-2}C_h)$ of $K$ in \eqref{eq:onset}
are nonnegative, the last one because $2\gamma^{-2}C_h>1$ by
\eqref{eq:b4-constants}. Hence
\[
 K\ge\exp(4),
 \qquad
 s\ge5,
 \qquad
 h\gamma^{-2}\le K,
 \qquad
 p_h\le K,
 \qquad
 \ln(2\gamma^{-2}C_h)\le K.
\]
The ceiling in \eqref{eq:onset} adds less than one and $Ks\ge1$, so
\[
 128Ks\le\eta_*\le128Ks+1\le129Ks.
\]
In particular $\eta_*\ge640\exp(4)>\exp(10)$, because $\exp(6)<404<640$, and
every $\eta\ge\eta_*$ satisfies $\eta\ge4$ and $\Lambda>10$. The upper bound
gives
\[
 \ln\eta_*\le\ln129+(s-1)+\ln s\le3s,
\]
where the last step uses $\ln129<5\le s$ and $\ln s\le s$. The function
$s\mapsto s-\ln(3s)$ has derivative $1-1/s>0$ for $s\ge5$ and the positive
value $5-\ln15$ at $s=5$, so $\ln(3s)\le s$. Therefore
\begin{equation}
 10<\ln\eta_*\le3s,
 \qquad
 0<\ln\ln\eta_*\le s.
 \label{eq:b4-onset-logs}
\end{equation}

\emph{The root condition.} By $h\gamma^{-2}\le K$ and
\eqref{eq:b4-onset-logs},
\[
 4h\gamma^{-2}\ln\eta_*\le12Ks<128Ks\le\eta_*.
\]
The function $\eta\mapsto\eta/\ln\eta$ has derivative
$(\ln\eta-1)/(\ln\eta)^{2}>0$ for $\eta>e$, so
$\eta/\Lambda\ge\eta_*/\ln\eta_*\ge4h\gamma^{-2}$ for every $\eta\ge\eta_*$.
This is the third condition, and dividing it by $h\gamma^{-2}=1/\theta_h$
gives $4\Lambda\le\eta\theta_h$.

\emph{The horizon condition.} For $\eta\ge\eta_*$ define
\[
 \Phi(\eta):=\frac\eta4-\ln\big(2\gamma^{-2}\eta\,C_h(\ln\eta)^{p_h}\big).
\]
At $\eta=\eta_*$, \eqref{eq:b4-onset-logs} together with $0\le p_h\le K$ and
$\ln(2\gamma^{-2}C_h)\le K$ gives
\[
\begin{aligned}
 \ln\big(2\gamma^{-2}\eta_*C_h(\ln\eta_*)^{p_h}\big)
 &=\ln(2\gamma^{-2}C_h)+\ln\eta_*+p_h\ln\ln\eta_*\\
 &\le K+3s+Ks
 \le3Ks
 <32Ks
 \le\frac{\eta_*}{4},
\end{aligned}
\]
where $K+3s+Ks\le3Ks$ because $K\le Ks$ and $3s\le Ks$, by $s\ge1$ and
$K>3$, and the last step is $\eta_*\ge128Ks$. Thus $\Phi(\eta_*)>0$. Moreover,
\[
 \Phi'(\eta)=\frac14-\frac1\eta-\frac{p_h}{\eta\ln\eta}
 >\frac14-\frac1{128}-\frac1{640}>0
 \qquad(\eta\ge\eta_*),
\]
because $\eta>\exp(10)>128$, $\eta\ge128Ks\ge640K\ge640p_h$, and
$\ln\eta>1$. Hence $\Phi(\eta)>0$ for every $\eta\ge\eta_*$, and
exponentiating gives
\[
 \frac{\gamma^{2}\exp(\eta/4)}{\eta}>2C_h\Lambda^{p_h}.
\]
Finally, $C_h\Lambda^{p_h}\ge1$ because $C_h>1$, $\Lambda>10$, and
$p_h\ge0$, so the definition \eqref{eq:horizon} and the elementary bound
$\lfloor u\rfloor>u-1$ give the strict inequality
\[
 \widehat T
 >\frac{\gamma^{2}\exp(\eta/4)}{\eta}-1
 >2C_h\Lambda^{p_h}-1
 \ge C_h\Lambda^{p_h}.
 \qedhere
\]
\end{proof}

We now combine the counts of the preceding parts.

\begin{proof}[Proof of Theorem~\ref{thm:general}]
Fix $\eta\ge\eta_h^0$, and let $0<\gamma<1/\sqrt2$ as in the theorem, so
that $\gamma^{-2}>2$. By Lemma~\ref{lem:b4-onset}, the hypotheses
$\eta\ge4$, $\Lambda>10$, and $4\Lambda\le\eta\theta_h$ of the results cited
below hold; the fourth condition $C_h\Lambda^{p_h}<\widehat T$ is used
only in the last step, where we pass from $U=\min\{\tau,\widehat T\}$
to $\tau$. Each of the steps $0,1,\ldots,U-1$ is step zero, a step of
$\mathcal T_-$, or a step of $\mathcal T_+\setminus\{0\}$, so
$U\le1+|\mathcal T_-|+|\mathcal T_+\setminus\{0\}|$, and we bound
$1+|\mathcal T_-|$ and $|\mathcal T_+\setminus\{0\}|$ separately.

\emph{Step zero and $\mathcal T_-$.} Lemma~\ref{lem:b1-basic} gives
$|\mathcal T_-|<4n\gamma^{-2}$. Since $n\gamma^{-2}>1$ and $h\gamma^{-2}>1$,
\begin{equation}
 1+|\mathcal T_-|<1+4n\gamma^{-2}<5n\gamma^{-2}
 \le5\gamma^{-4}h\big(n+h\ln(4\gamma^{-2})\big).
 \label{eq:b4-negative-count}
\end{equation}

\emph{Root intervals.} By the root condition, Lemma~\ref{lem:b2-roots}
partitions $\mathcal T_+\setminus\{0\}$ into $M$ consecutive root
intervals, each carrying one reference direction. By
$|\mathcal T_-|<4n\gamma^{-2}$, $1/\theta_h=h\gamma^{-2}$, and
$\lceil u\rceil<u+1$, their number satisfies
\[
 M\le|\mathcal T_-|+1+\big\lceil4h\gamma^{-2}\ln(4\gamma^{-2})\big\rceil
 <4\gamma^{-2}\big(n+h\ln(4\gamma^{-2})\big)+2
 \le5\gamma^{-2}\big(n+h\ln(4\gamma^{-2})\big).
\]
The last step holds because $\ln(4\gamma^{-2})>\ln8>2$ gives
$n+h\ln(4\gamma^{-2})>5$ and therefore
$\gamma^{-2}(n+h\ln(4\gamma^{-2}))>10>2$. Each root interval $J$ is refined
recursively by Lemma~\ref{lem:b2-refine} into the tree of
Appendix~\ref{app:packing}, and Proposition~\ref{prop:b3-root-count} gives
\[
 |J|\le\big(3\nu_h+12h\gamma^{-2}\big)F_h,
 \qquad
 F_h:=\frac{(6h\gamma^{-2}\Lambda)^{\nu_h}}{\nu_h!}.
\]
Since $3\nu_h<6\gamma^{-2}\le3h\gamma^{-2}$ by \eqref{eq:depth-range} and
$h\ge2$, the first factor is at most $15h\gamma^{-2}$. Multiplying by the
bound on $M$,
\begin{equation}
 |\mathcal T_+\setminus\{0\}|
 \le75\gamma^{-4}h\big(n+h\ln(4\gamma^{-2})\big)F_h.
 \label{eq:b4-root-total}
\end{equation}

\emph{The structural count.} The factor $F_h$ is at least one. This is an
equality when $\nu_h=0$. When $\nu_h\ge1$, we have
$6h\gamma^{-2}\Lambda>2\gamma^{-2}>\nu_h$ by $6h\Lambda>2$ and
\eqref{eq:depth-range}, and $\nu_h!\le\nu_h^{\nu_h}$, so
$F_h\ge(6h\gamma^{-2}\Lambda/\nu_h)^{\nu_h}>1$. The right side of
\eqref{eq:b4-negative-count} is therefore at most
$5\gamma^{-4}h(n+h\ln(4\gamma^{-2}))F_h$, and adding
\eqref{eq:b4-root-total} gives, by
$U\le1+|\mathcal T_-|+|\mathcal T_+\setminus\{0\}|$ and
\eqref{eq:prefactor},
\begin{equation}
 U\le(5+75)\gamma^{-4}h\big(n+h\ln(4\gamma^{-2})\big)F_h
 =C_h\Lambda^{\nu_h}\,\frac{(6h\gamma^{-2})^{\nu_h}}{\nu_h!}.
 \label{eq:b4-structural}
\end{equation}
At $h=h_*$ this is \eqref{eq:structural}.

\emph{Absorbing the fixed coefficient.} If $\nu_h=0$, then $p_h=0$ and
\eqref{eq:b4-structural} reads $U\le C_h=C_h\Lambda^{p_h}$. Assume from now
on that $\nu_h\ge1$, and recall $b_h=\max\{2,\nu_h\}$. One term of the
exponential series gives $\exp(\nu_h)\ge\nu_h^{\nu_h}/\nu_h!$, that is,
$\nu_h!\ge(\nu_h/e)^{\nu_h}$, and $e<3$ gives
\begin{equation}
 \frac{(6h\gamma^{-2})^{\nu_h}}{\nu_h!}
 \le\bigg(\frac{6eh\gamma^{-2}}{\nu_h}\bigg)^{\nu_h}
 \le\bigg(\frac{18h\gamma^{-2}}{\nu_h}\bigg)^{\nu_h}.
 \label{eq:b4-factorial-absorption}
\end{equation}
Multiplying the first inequality in \eqref{eq:b4-ceiling} by
$\ln b_h\ge\ln2>0$ and exponentiating,
\begin{equation}
 b_h^{\,p_h-\nu_h}\ge\bigg(\frac{18h\gamma^{-2}}{\nu_h}\bigg)^{\nu_h}.
 \label{eq:b4-absorption-power}
\end{equation}
At $h=h_*$, \eqref{eq:b4-factorial-absorption} and
\eqref{eq:b4-absorption-power} are \eqref{eq:absorb}. Converting the power
of $b_h$ into a power of $\Lambda$ requires two ranges of $\Lambda$. To keep
the stepsize threshold polynomial, the structural count is used when
$\Lambda$ is at least the depth, and the coefficient budget of
Lemma~\ref{lem:b1-basic} otherwise.

\emph{Case 1: $\Lambda\ge b_h$.} By \eqref{eq:b4-structural},
\eqref{eq:b4-factorial-absorption}, \eqref{eq:b4-absorption-power}, and
$p_h-\nu_h\ge0$,
\[
 U\le C_h\Lambda^{\nu_h}\bigg(\frac{18h\gamma^{-2}}{\nu_h}\bigg)^{\nu_h}
 \le C_h\Lambda^{\nu_h}b_h^{\,p_h-\nu_h}
 \le C_h\Lambda^{\nu_h}\Lambda^{p_h-\nu_h}
 =C_h\Lambda^{p_h}.
\]

\emph{Case 2: $10<\Lambda<b_h$.} Then $b_h>2$, so $b_h=\nu_h$ and
$\Lambda<\nu_h$. Here the structural count is replaced by a crude count.
Lemma~\ref{lem:b1-basic} gives $\mu_t>1/2$ at every $t\in\mathcal T_+$ and,
in \eqref{eq:b1-budget}, $\sum_{t<U}\mu_t\le2\eta\gamma^{-2}$. Since every
$\mu_t$ is positive,
\[
 \frac12|\mathcal T_+|
 \le\sum_{t\in\mathcal T_+}\mu_t
 \le\sum_{t<U}\mu_t
 \le2\eta\gamma^{-2},
\]
so $|\mathcal T_+|\le4\gamma^{-2}\exp(\Lambda)\le C_h\exp(\Lambda)$ by
\eqref{eq:b4-constants}, and
$|\mathcal T_-|<4n\gamma^{-2}\le C_h$ by \eqref{eq:b1-budget} and
\eqref{eq:b4-constants}. Since
$\{0,1,\ldots,U-1\}=\mathcal T_+\mathbin{\dot\cup}\mathcal T_-$,
\begin{equation}
 U=|\mathcal T_+|+|\mathcal T_-|\le C_h\big(1+\exp(\Lambda)\big).
 \label{eq:b4-crude}
\end{equation}
Since $\Lambda>e$, we have $\ln\Lambda>1$ and hence
$\exp(\Lambda)\le\exp(\Lambda\ln\Lambda)=\Lambda^{\Lambda}\le\Lambda^{\nu_h}$,
the last step by $1<\Lambda<\nu_h$. Using $\Lambda^{\nu_h}\ge1$,
$\Lambda>2$, and $p_h\ge\nu_h+1$, which was observed after
\eqref{eq:exponent}, \eqref{eq:b4-crude} gives
\[
 U\le C_h\big(1+\Lambda^{\nu_h}\big)
 \le2C_h\Lambda^{\nu_h}
 <C_h\Lambda^{\nu_h+1}
 \le C_h\Lambda^{p_h}.
\]
Thus $U\le C_h\Lambda^{p_h}$ in every case, including $\nu_h=0$. Because of
the second case, the threshold $\eta_h^0$ never needs $\Lambda\ge\nu_h$, which would be an
exponential requirement on the stepsize.

\emph{Removing the horizon.} The fourth condition in \eqref{eq:b4-onset}
now gives the strict comparison
\[
 U\le C_h\Lambda^{p_h}<\widehat T.
\]
Hence $U\ne\widehat T$, and $U=\min\{\tau,\widehat T\}$
forces $U=\tau$. Therefore
$\tau=U\le C_h\Lambda^{p_h}=C_h(\ln\eta)^{p_h}$, which is the bound of
Theorem~\ref{thm:general}.
\end{proof}

By \eqref{eq:selected}, Theorem~\ref{thm:general} at $h=h_*\ge4$ is the
bound \eqref{eq:main-bound} of Theorem~\ref{thm:main}(i) for
$0<\gamma<1/\sqrt2$.

The structural count \eqref{eq:b4-structural} is itself a bound on the
phase transition time, in which the factorial saving of
Appendix~\ref{app:packing} is still visible. It coincides with the bound of
Theorem~\ref{thm:general} when $\nu_h=0$, and by
\eqref{eq:b4-factorial-absorption} and \eqref{eq:b4-absorption-power} it is
at least as sharp when $\nu_h\ge1$ and $\Lambda\ge\max\{2,\nu_h\}$.

\begin{proposition}[Factorial bound on the phase transition time]
\label{prop:b4-factorial}
Suppose Assumption~\ref{ass:data} holds with $0<\gamma<1/\sqrt2$. For every
integer $h\ge2$ and every $\eta\ge\eta_h^0$,
\begin{equation}
 \tau\le C_h\,\frac{(6h\gamma^{-2}\ln\eta)^{\nu_h}}{\nu_h!}.
 \label{eq:b4-factorial-bound}
\end{equation}
\end{proposition}

\begin{proof}
The proof of Theorem~\ref{thm:general} established $U=\tau$, and
substituting this identity into \eqref{eq:b4-structural} gives
\eqref{eq:b4-factorial-bound}.
\end{proof}

\end{document}